\documentclass[conference]{IEEEtran}
\IEEEoverridecommandlockouts
\usepackage{cite}
\usepackage{amsfonts}
\usepackage[linesnumbered,lined,commentsnumbered,ruled,vlined]{algorithm2e}
\usepackage[]{algpseudocode}
\usepackage{graphicx}
\usepackage{textcomp}
\def\BibTeX{{\rm B\kern-.05em{\sc i\kern-.025em b}\kern-.08em
    T\kern-.1667em\lower.7ex\hbox{E}\kern-.125emX}}

\usepackage[T1]{fontenc}    
\usepackage{hyperref}       
\usepackage{url}            
\usepackage{booktabs}       
\usepackage{amsfonts}       
\usepackage{nicefrac}       
\usepackage{microtype}      
\usepackage[labelfont=bf]{caption}
\usepackage{lipsum}
\usepackage{multicol}
\usepackage{dsfont}
\usepackage{amsmath,amssymb} 
\usepackage{color}
\usepackage{pst-plot}
\usepackage{pst-math}
\usepackage{pstricks}
\usepackage{pgfplots,filecontents}
\usepackage{subfigure}
\pgfplotsset{compat=1.7}
\usepackage{relsize}
\usepackage{amsthm}
\usepackage{mathtools}
\usepackage{enumitem}
\newtheorem{theorem}{Theorem}[section]

\newtheorem{lemma}[theorem]{Lemma}

\newlist{myitemize}{enumerate}{10}
\setlist[myitemize]{label*=\arabic*.,nosep,leftmargin=*}

\DeclareMathOperator{\A}{\mathbf{A}}
\DeclareMathOperator{\Ll}{\mathbf{L}}
\DeclareMathOperator{\D}{\mathbf{D}}
\DeclareMathOperator{\Dn}{\mathbf{D^{-\frac{1}{2}}}}
\DeclareMathOperator{\Pp}{\mathbf{P}}
\DeclareMathOperator{\Qq}{\mathbf{Q}}
\DeclareMathOperator{\I}{\mathbf{I}}

\DeclareMathOperator{\gt}{g_{\mathbf{\theta}}}
\DeclareMathOperator{\nr}{\mathbf{N_{n}}}

\DeclareMathOperator{\Uu}{\mathbf{U}}
\DeclareMathOperator{\UT}{\mathbf{U^{\intercal}}}
\DeclareMathOperator{\g}{\mathbf{g}}
\DeclareMathOperator{\LBD}{\mathbf{\Lambda}}

\DeclareMathOperator{\N}{\mathbf{N}}

\begin{document}

\title{Rational Neural Networks for Approximating Jump Discontinuities of Graph Convolution Operator\\
}

\author{\IEEEauthorblockN{
Zhiqian Chen\IEEEauthorrefmark{1},
Feng Chen\IEEEauthorrefmark{2}, 
Rongjie Lai\IEEEauthorrefmark{3},
Xuchao Zhang\IEEEauthorrefmark{1} and
Chang-Tien Lu\IEEEauthorrefmark{1}}
\IEEEauthorblockA{\IEEEauthorrefmark{1}Computer Science Department, Virginia Tech\\Email:\{czq,xuczhang,ctlu\}@vt.edu}
\IEEEauthorblockA{\IEEEauthorrefmark{2}Department of Computer Science, University at Albany\\Email:\{fchen5\}@albany.edu}
\IEEEauthorblockA{\IEEEauthorrefmark{3}Department of Mathematics, Rensselaer Polytechnic Institute\\Email:\{lair\}@rpi.edu}
}
\maketitle
\begin{abstract}
    For node level graph encoding, a recent important state-of-art method is the graph convolutional networks (GCN), which nicely integrate local vertex features and graph topology in the spectral domain. However, current studies suffer from several drawbacks: (1) graph CNNs relies on Chebyshev polynomial approximation which results in oscillatory approximation at jump discontinuities; (2) Increasing the order of Chebyshev polynomial can reduce the oscillations issue, but also incurs unaffordable computational cost; (3) Chebyshev polynomials require degree $\Omega$(poly(1/$\epsilon$)) to approximate a jump signal such as $|x|$, while rational function only needs $\mathcal{O}$(poly log(1/$\epsilon$))\cite{liang2016deep,telgarsky2017neural}. However, it's non-trivial to apply rational approximation without increasing computational complexity due to the denominator.
    
    In this paper, the superiority of rational approximation is exploited for graph signal recovering. RatioanlNet is proposed to integrate rational function and neural networks. We show that rational function of eigenvalues can be rewritten as a function of graph Laplacian, which can avoid multiplication by the eigenvector matrix. Focusing on the analysis of approximation on graph convolution operation, a graph signal regression task is formulated. Under graph signal regression task, its time complexity can be significantly reduced by graph Fourier transform. To overcome the local minimum problem of neural networks model, a relaxed Remez algorithm is utilized to initialize the weight parameters. Convergence rate of RatioanlNet and polynomial based methods on jump signal is analyzed for a theoretical guarantee. The extensive experimental results demonstrated that our approach could effectively characterize the jump discontinuities, outperforming competing methods by a substantial margin on both synthetic and real-world graphs.
    
\end{abstract}

\section{Introduction}
Effective information analysis generally boils down to the geometry of the data represented by a graph. Typical applications include social networks\cite{lazer2009life}, transportation networks\cite{bell1997transportation}, spread of epidemic disease\cite{newman2002spread}, brain's neuronal networks\cite{marx2012high}, gene data on biological regulatory networks\cite{davidson2002genomic}, telecommunication networks\cite{drew2008diagnosing}, knowledge graph\cite{lin2015learning}, which are lying on non-Euclidean graph domain. To describe the geometric structures, graph matrices such as adjacency matrix or graph Laplacian can be employed to reveal latent patterns.

\begin{figure}[h]
	\begin{subfigure}
		\centering
		\includegraphics[width=1.68in]{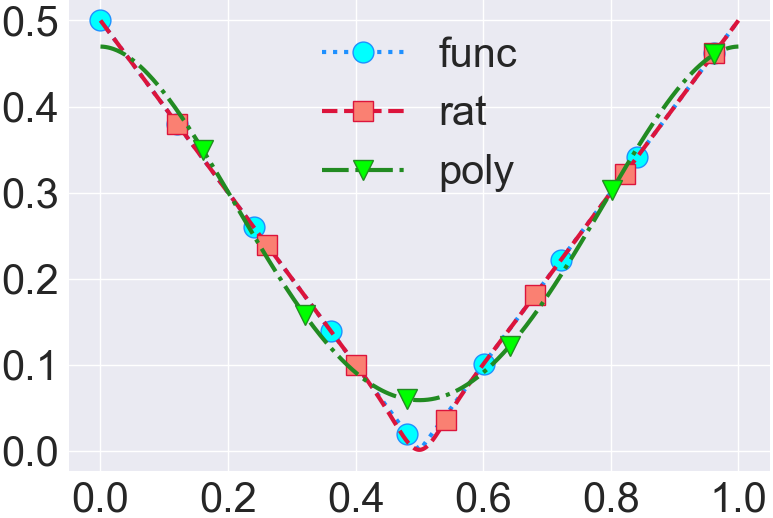}
	\end{subfigure}
	\begin{subfigure}
		\centering
		\includegraphics[width=1.68in]{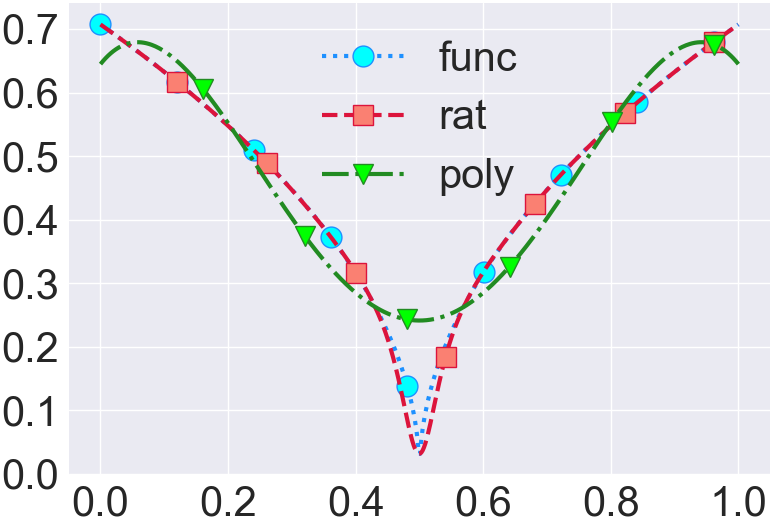}
	\end{subfigure}\\
	\begin{subfigure}
		\centering
		\includegraphics[width=1.68in]{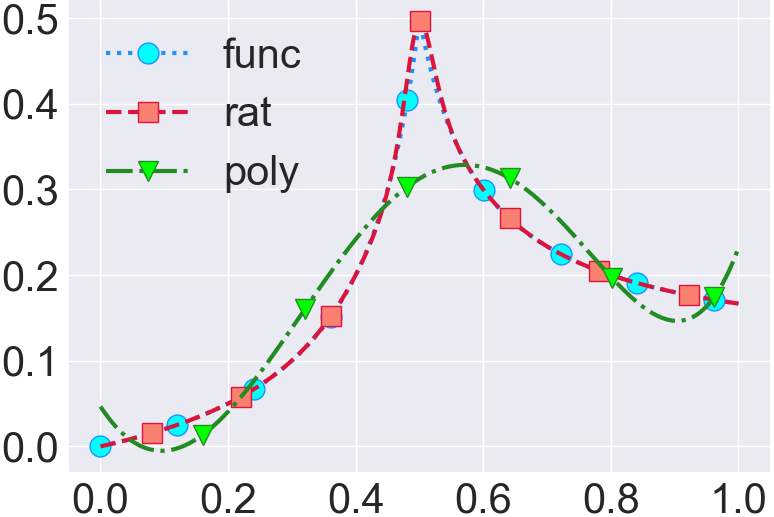}
	\end{subfigure}
	\begin{subfigure}
		\centering
		\includegraphics[width=1.68in]{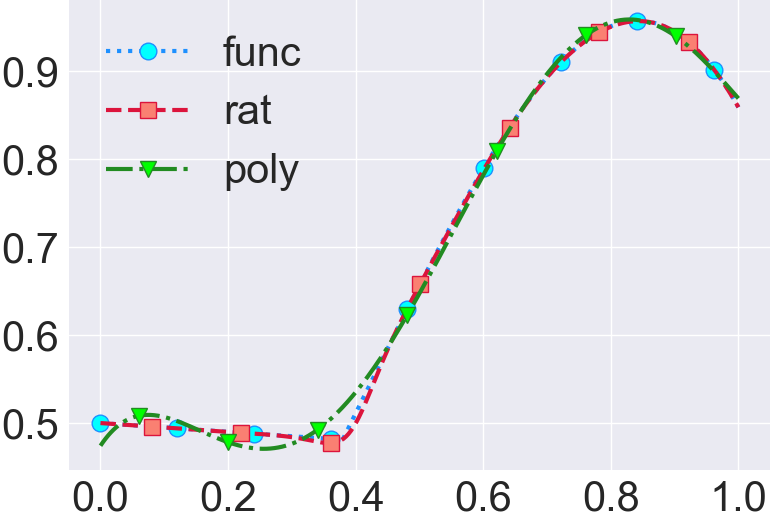}
	\end{subfigure}
	\caption{Rational(rat) and polynomial(poly) approximation for several jump functions(func). 
		Top left:$\sqrt{|x-0.5|}$; top right: $|x-0.5|$; bottom left: $\frac{x}{10|x - 0.5|+1}$; bottom right:    $max(0.5, sin(x+x^{2}))-\frac{x}{20}$}
	\label{fig:fit_example}
\end{figure}
In recent years, many problems are being revisited with deep learning tools. Convolutional neural networks(ConvNets) emerging in recent years are at the heart of deep learning, and the most prominent strain of neural networks in research. ConvNets have revolutionized computer vision\cite{krizhevsky2012imagenet}, natural language processing\cite{collobert2008unified}, computer audition\cite{graves2013speech}, reinforcement learning\cite{mnih2015human,silver2016mastering}, and many other areas. However, ConvNets are designed for grid data such as image, which belongs to the Euclidean domain. Graph data is non-Euclidean which makes it difficult to employ typical ConvNets.
To bridge the gap, Bruna et al. \cite{bruna2013spectral}\cite{henaff2015deep} generalized spectral convolutional operation which requires expensive steps of spectral decomposition and matrix multiplication. Hammond et al.\cite{hammond2011wavelets} first introduced truncated Chebyshev polynomial for estimating wavelet in graph signal processing. Based on this polynomial approximation, Defferrard et al.\cite{defferrard2016convolutional} designed ChebNet which contains a novel neural network layer for the convolution operator in the spectral domain. Kipf and Welling\cite{kipf2017semi} simplified ChebNet by assuming the maximum of eigenvalues is 2 and fixing the order to 1, which boosts both effectiveness and efficiency. Li et al.\cite{DBLP:journals/corr/abs-1801-07606} found that this simplified ChebNet is an application of Laplacian smoothing, which implies that current studies are only effective on the smooth signal.

Current studies on graph ConvNet heavily rely on polynomial approximation, which makes it difficult to estimate jump signals.  Fig. \ref{fig:fit_example} shows the behaviors of polynomial and rational function on jump discontinuity: rational approximation fits the functions considerably better than polynomials. It is widely recognized that \textbf{(1)} polynomial approximation suffers from Gibbs phenomenon, which means polynomial function oscillate and overshoot near discontinuities\cite{trefethen2013approximation}; \textbf{(2)} Applying a higher order of polynomials could dramatically reduce the oscillation, but also incurs an expensive computational cost. \textbf{(3)} Polynomials require degree $\Omega$(poly(1/$\epsilon$)) to approximate functions near singularities and on an unbounded domain, while rational functions only need $\mathcal{O}$(poly log(1/$\epsilon$)) to achieve $\epsilon$-close\cite{liang2016deep,telgarsky2017neural}. However, it is non-trivial to apply rational approximation. Polynomial-based method can easily transfer the function on eigenvalues to the same function on graph Laplacian so that matrix multiplication by eigenvector can be avoided. It is not easy for rational approximation to do so due to the additional denominator.

In this paper, the advantage of the rational function in approximation is transferred to spectral graph domain. Specifically, we propose a rational function based neural networks(RationalNet), which can avoid matrix multiplication by eigenvectors. To alleviate the local minimum problem of the neural network, a relaxed Remez algorithm is employed for parameter initialization. Our theoretical analysis shows that rational functions converge much faster than polynomials on jump signal. In a nutshell, the key innovations are:

\begin{itemize}
	\item \textbf{Propose a neural network model based on rational function for recovering jump discontinuities}: To estimate the jump signal, our proposed method integrates rational approximation and spectral graph operation to avoid matrix multiplication by eigenvectors. For graph signal regression task, expensive matrix inversion can be circumvented by graph Fourier transform.
	\item \textbf{Develop an efficient algorithm for model parameters optimization}: Remez algorithm is theoretically optimal, but it is often not practical especially when approximating discrete signal. To alleviate this issue, the stopping rules of Remez algorithm are relaxed to initialize the neural networks parameters.
	\item \textbf{Provide theoretical analysis for the proposed method on jump signal}: For understanding the behaviors of polynomial and rational function on jump discontinuities, a uniform representation is proposed to analyze convergence rate regarding the order number theoretically.
	\item \textbf{Conducting extensive experiments for performance evaluations}\footnote{The code and datasets will be released after the acceptance for replication}: The proposed method was evaluated on synthetic and real-world data. Experimental results demonstrate that the proposed approach runs efficiently and consistently outperforms the best of the existing methods.
\end{itemize}

The rest of the paper is organized as follows. Section \ref{sec:related_work} reviews existing work in this area. Necessary preliminary is presented in section \ref{sec:prelim}. Section \ref{sec:method} elaborates a rational function based model for graph signal regression task. Section \ref{sec:algorithm} presents algorithm description and theoretical analysis. In Section \ref{sec:evaluation}, experiments on synthetic and real-world data are analyzed. This paper concludes by summarizing the study's important findings in Section \ref{sec:conclusion}.

\section{Related Work}\label{sec:related_work}
The proposed method draws inspiration from the field of approximation theory, spectral graph theory and recent works using neural networks. In what follows, we provide a brief overview of related work in these fields.

\subsection{Approximation theory}
In mathematics, approximation theory is concerned with how functions can best be approximated with simpler functions, and with quantitatively characterizing the errors introduced thereby. One problem of particular interest is that of approximating a function in a computer mathematical library, using operations that can be performed on the computer or calculator, such that the result is as close to the actual function as possible. This is typically done with polynomial or rational approximations. Polynomials are familiar and comfortable, but rational functions seem complex and specialized, and rational functions are more powerful than polynomials at approximating functions near singularities and on unbounded domains. Basic properties of rational function are described in books of complex analysis\cite{ahlfors1953complex,trefethen2013approximation,pachon2010algorithms,powell1981approximation,cohen2011numerical,petrushev2011rational,achieser2013theory,ziegel1987numerical,boyd2001chebyshev,mason2002chebyshev,remez1934determination}. 

\subsection{Spectral graph theory}
Spectral graph theory is the study of the properties of a graph in relationship to the characteristic polynomial, eigenvalues, and eigenvectors of matrices associated with the graph, such as its adjacency matrix or Laplacian matrix\cite{chung1997spectral,grone1990laplacian,das2004laplacian}. Many graphs and geometric convolution methods have been proposed recently. The spectral convolution methods (\cite{bruna2013spectral,defferrard2016convolutional,kipf2017semi,bronstein2017geometric}) are the mainstream algorithm developed as the graph convolution methods. Because their theory is based on the graph Fourier analysis (\cite{shuman2013emerging, shuman2016vertex}). The polynomial approximation is firstly proposed by \cite{hammond2011wavelets}. Inspired by this, graph convolutional neural networks (GCNNs) (\cite{defferrard2016convolutional}) is a successful attempt at generalizing the powerful convolutional neural networks (CNNs) in dealing with Euclidean data to modeling graph-structured data. Kipf and Welling proposed a simplified type of GCNNs\cite{kipf2017semi}, called graph convolutional networks (GCNs). The GCN model naturally integrates the connectivity patterns and feature attributes of graph-structured data and outperforms many state-of-the-art methods significantly. 
Li et al.\cite{DBLP:journals/corr/abs-1801-07606} found that GCN is actual an application of Laplacian smoothing, which is inconsistent with GCN's motivation. In sum, this thread of work calculates the average of vertexes within Nth-order neighbors.

In this paper, we focus on the effectiveness of approximation technique on graph signal. Under graph signal regression task, the superiority of rational function beyond polynomial function is analyzed, and a rational function based neural network is proposed. 

\section{Preliminaries}\label{sec:prelim}

We focus processing graph signals defined on undirected graphs $\mathcal{G} = (\mathcal{V}, \mathcal{E}, \mathcal{W})$, where $\mathcal{V}$ is a set of n vertexes, $\mathcal{E}$ represents edges and
$\mathcal{W} = [w_{ij}] \in \{0,1\}^{n\times n}$ is an unweighted adjacency matrix. A signal $x : \mathcal{V} \rightarrow \mathbb{R}$ defined on the nodes may be regarded as a vector $x \in \mathbb{R}^{n}$.
Combinatorial graph Laplacian\cite{chung1997spectral} is defined as $\Ll= D-\mathcal{W} \in \mathbb{R}^{n\times n}$ where $D$ is degree matrix.

As $\Ll$ is a real symmetric positive semidefinite matrix, it has a
complete set of orthonormal eigenvectors and their associated ordered real nonnegative eigenvalues identified as the frequencies of the graph. The Laplacian is diagonalized by the Fourier basis $\UT$: $\Ll = \Uu \Lambda \UT$ where $\Lambda$ is  the diagonal matrix whose diagonal elements are the corresponding eigenvalues, i.e., ${\displaystyle \Lambda _{ii}=\lambda _{i}}$. The graph Fourier transform of a signal $x\in \mathbb{R}^{n}$ is defined as $\hat{x}=\UT x \in \mathbb{R}^{n}$ and its inverse as $x=\Uu \hat{x}$\cite{shuman2013emerging, shuman2016vertex, zhu2012approximating}. To enable the formulation of fundamental operations such as filtering in the vertex domain, the convolution operator on graph is defined in the Fourier domain such that $f_{1}*f_{2}=\Uu \left[\left(\UT f_{1} \right) \odot \left(\UT f_{2}\right)\right]$, where $\odot$ is the element-wise product, and $f_{1}/f_{2}$ are two signals defined on vertex domain. It follows that a vertex
signal $f_{2}=x$ is filtered by spectral signal $\hat{f_{1}}=\UT f_{1}=\g$ as:
\begin{equation*}
\small
\g * x = \Uu \left[\g(\LBD)\odot \left(\UT f_{2}\right)\right] = \Uu \g(\LBD) \UT x.
\end{equation*}
Note that a real symmetric matrix $\Ll$ can be decomposed as $\Ll=\Uu \LBD \Uu^{-1} =  \Uu \LBD \UT$ since $\Uu^{-1}=\UT$ . D. K. Hammond et al. and Defferrard et al.\cite{hammond2011wavelets, defferrard2016convolutional} apply polynomial approximation on spectral filter $\g$ so that:
\begin{align*}
\small
& \g * x = \Uu \g (\LBD) \UT x &&\\
\approx&\Uu \sum_{k}^{}\theta_{k} T_{k}(\tilde{\LBD}) \UT x && {(\tilde{\LBD}=\frac{2}{\lambda_{max}}\LBD-\I_{\N})}\\
=&\sum_{k}^{}\theta_{k} T_{k}(\tilde{\Ll}) x &&{(\Uu\LBD^{k}\UT=(\Uu\LBD\UT)^{k})}\\
\end{align*}
Kipf and Welling\cite{kipf2017semi} simplifies it by:
\begin{align*}
\small
&\g * x &&\\
\approx& \theta_{0}\I_{\N}x+\theta_{1}\tilde\Ll x &&({\scriptstyle\text{expand to 1st order)}}\\
=&\theta_{0}\I_{\N}x+\theta_{1}(\frac{2}{\lambda_{max}}\Ll-\I_{\N})) x &&{\scriptstyle(\tilde\Ll=\frac{2}{\lambda_{max}}\Ll-\I_{\N}))} \\
=&\theta_{0}\I_{\N}x+\theta_{1}(\Ll-\I_{\N})) x &&{\scriptstyle(\lambda_{max}=2)} \\
=&\theta_{0}\I_{\N}x-\theta_{1} \Dn\A\Dn x &&{\scriptstyle(\Ll=\I_{\N}-\Dn\A\Dn)} \\
=&\theta_{0}(\I_{\N} + \Dn\A\Dn) x &&{\scriptstyle(\theta_{0}=-\theta_{1})} \\
=&\theta_{0}(\tilde\D^{-\frac{1}{2}}\tilde\A\tilde\D^{-\frac{1}{2}}) x &&{\scriptstyle(\text{renormalization}: \tilde\A=\A+\I_{\N},}\\
&&&{\scriptstyle \tilde\D_{ii}=\sum_{j} \A_{ij})}, \\
\end{align*}rewrite the above GCN in matrix form: $\gt*X \approx(\tilde\D^{-\frac{1}{2}}\tilde\A\tilde\D^{-\frac{1}{2}}) X\Theta$, which leads to \textit{symmetric normalized Laplacian} with raw feature. GCN has been analyzed GCN in \cite{DBLP:journals/corr/abs-1801-07606} using smoothing Laplacian\cite{taubin1995signal}:
$y= (1-\gamma) x_{i} + \gamma \sum_{j}\frac{\tilde a_{ij}}{d_{i}}x_{j} 
= x_{i} - \gamma (x_{i}-\sum_{j}\frac{\tilde a_{ij}}{d_{i}}x_{j}),
$where $\gamma$ is a weight parameter between the current vertex $x_{i}$ and the features of its neighbors $x_{j}$, $d_{i}$ is degree of $x_{i}$, and $y$ is the smoothed Laplacian. This smoothing Laplacian has a matrix form：
\begin{align*}
Y&= x -\gamma \tilde\D^{-1}\tilde\Ll x&& \\
&= (\I_{\N}-\tilde\D^{-1}\tilde\Ll )x && (\gamma=1)\\
&= (\I_{\N}-\tilde\D^{-1}(\tilde\D-\tilde\A))x && (\tilde\Ll=\tilde\D-\tilde\A)\\
&= \tilde\D^{-1}\tilde\A x.&& \label{eq:smooth_laplacian}
\end{align*}

The above formula is \textit{random walk normalized Laplacian} as a counterpart of \textit{symmetric normalized Laplacian}. Therefore, GCN is nothing but a first-order Laplacian smoothing which averages neighbors of each vertex.

\section{Model description}\label{sec:method}
This section formally defines the task of graph signal recovering and then describes our proposed RationalNet which aims to characterize the jump signal in spectral domain.

\subsection{Problem Setting}\label{sec:ps}
All the related works integrate graph convolution estimator and fully-connected neural layers for a classification task. This classification can be summarized as:
\begin{equation}
Y=f(\mathcal{G}, x)\Theta,
\label{classfication_task}
\end{equation}where $\Theta$ indicates the parameters of normal neural network layers connecting the output of $f$ and the label $Y$, such as fully-connected layers and softmax layers for classification. And $f$ is a neural network layer implemented by approximation techniques. However, whether the success is due to the neural networks($\Theta$) or the convolution approximation method($f$) remains unknown. To focus on the analysis of approximation on $f$, a graph signal regression task is proposed to evaluate the performance of the convolution approximators $f$. Regression task directly compares the label and the output of $f$, removing the distraction of $\Theta$.

Given a graph $\mathcal{G}$, raw feature $x$, and training signal on the part of vertexes, $Y_{train}$, our goal is to recover signal values, $Y_{test}$, on test nodes. Formally, we want to find a $f(\cdot)$ so that: 

\begin{equation*}
Y=f(\mathcal{G}, x).
\end{equation*}
If the raw features are good enough for the regression task, whether the effectiveness is due to $f$ or $x$ is difficult to verify. Therefore, one reasonable option for $x$ is the uniform signal in spectral domain. Specifically, $x=\sum_{i}^{} \Uu_{i}$ and $\hat{x}=\UT x=\mathds{1}=\{1,1,...,1\}$, which means that $x$ represents eigenbasis of graph structure in spectral domain. Each entry of vector $\hat{x}_{i}$ indicates one eigenvector in the spectral domain. The physical meaning of the convolution operation is how to select eigenbasis in spectral domain to match the graph signal $Y$. Representing $\mathcal{G}$ with graph Laplacian, the regress task can be rewritten as:

\begin{equation}
Y=f(\Ll, \Uu)=\Uu\gt(\LBD)\UT \sum_{i}\Uu,
\label{regression_task}
\end{equation}where $\gt$ is the spectral filter to approximate. 


\subsection{RationalNet}
Similar to polynomial approximation on graph domain such as ChebNet\cite{defferrard2016convolutional} or GCN\cite{kipf2017semi}, RationalNet approximates the spectral filter by a widely used type of rational function, i.e., Pad\'{e} approximator, which is defined as:
\begin{equation}
\small
R(x)=\frac{\sum_{i=0}^{m} \psi_{i}x^{i}}{\sum_{j=0}^{n}\phi_{j}x^{j}}, \phi_{0}=1, \phi_{j},\psi_{i}\in \mathbb{R}.
\label{eq:rat_define}
\end{equation}
 Applying to graph convolution operator, we have:

\begin{flalign*}
\small
&\gt * x=\Uu \gt \UT x && (\text{convolution theorem})\\
\approx&\Uu \frac{{\mathlarger{\sum}}_{i=0}^{m}\psi_{i} \tilde{\LBD}^{i}}{1+{\mathlarger{\sum}}_{j=1}^{n}\phi_{j} \tilde{\LBD}^{j}}  \UT x &&  {(\tilde{\LBD}=\frac{\LBD}{\lambda_{max}})} \\
=&\Uu \frac{\Pp(\LBD)}{\Qq(\LBD)} \UT x, &&(\text{define P and Q})
\end{flalign*}where $\LBD$ represents a diagonal matrix whose entries are eigenvalues, $\gt=R$, and $\theta=\{\psi, \phi\}$. The division $\frac{P}{Q}$ is element-wise. The inverse of matrix Q(x) is equivalent to applying reciprocal operation on its diagonal entries, so the equation can be rewritten as:
\begin{equation*}
\Uu \Pp(\LBD) \Qq(\LBD)^{-1} \UT x.
\end{equation*}Applying matrix rules, it's easy to have:
\begin{flalign*}
\small
=&\Uu \Pp(\LBD) \UT \Uu \Qq(\LBD)^{-1} \UT x &&{\scriptstyle(\UT=\Uu^{-1},\UT\Uu=\I_{\N})} \\
=&\big[\Uu \Pp(\LBD) \UT\big] \big[\Uu\Qq(\LBD)^{-1} \UT\big] x && \\
=&\big[\Pp(\Uu \LBD \UT)\big] \big[\Uu\Qq(\LBD)^{-1} \UT\big] x &&{\scriptstyle(\Uu\LBD^{k}\UT=(\Uu\LBD\UT)^{k})} \\
=&\big[\Pp(\Uu \LBD \UT)\big] \big[(\Qq(\LBD)\Uu^{-1})^{-1} \UT\big] x &&{\scriptstyle(A^{-1}B^{-1}=(BA)^{-1})} \\
=&\big[\Pp(\Uu \LBD \UT)\big] \big[(\Uu\Qq(\LBD)\Uu^{-1})^{-1}\big] x &&{\scriptstyle(\UT = \Uu^{-1})} \\
=&\big[\Pp(\Uu \LBD \UT)\big] \big[(\Uu\Qq(\LBD)\UT)^{-1}\big] x &&{\scriptstyle(\UT = \Uu^{-1})} \\
=&\big[\Pp(\Uu \LBD \UT)\big] \big[(\Qq(\Uu\LBD\UT))^{-1}\big] x &&{\scriptstyle(\Uu\LBD^{k}\UT=(\Uu\LBD\UT)^{k})} \\
\end{flalign*} Since $\Uu\LBD^{k}\UT=(\Uu\LBD\UT)^{k}$, we can rewrite the equation above as:
\begin{equation}
\small
\gt * x = \Pp(\Ll) \Qq(\Ll)^{-1} x,
\label{eq:ratnet}
\end{equation} 
where $\Pp (x)=\sum_{i=0}^{m} \psi_{i}x^{i}$ and $\Qq (x)=\sum_{j=0}^{n} \phi_{j}x^{j}$. Note that $\Pp(\Ll) \Qq(\Ll)^{-1} x=\Qq(\Ll)^{-1} \Pp(\Ll)  x$ in our case. Based on polynomial approximation, RationalNet only adds a inverse polynomial function $\Qq(\Ll)^{-1}$. Therefore, polynomial approximation (GCN/ChebNet) on graph is a special case of RationalNet when $\Qq(\Ll)^{-1}=\I$. 

Computing inverse of matrix is still of high complexity $\mathcal{O}(n^{3})$ (Gauss–Jordan elimination method) in Eq. \ref{eq:ratnet}, especially for large matrix. This can be avoided by transferring vertex graph signal and raw features into spectral domain. Therefore, the Eq. \ref{eq:regress} can be rewritten as:
\begin{equation}
\hat{Y}=\frac{\Pp(\LBD)}{\Qq(\LBD)} \hat{x},
\label{eq:ratnet2}
\end{equation}where $\hat{Y}=\UT Y$ is the graph Fourier transform of graph signal, and  $\hat{x}=\UT x$ is the graph Fourier transform of raw feature. By this step, we only need compute reciprocal of eigenvalues, rather than computing matrix multiplication and  inversion at each update. Eq. 2 can be obtained via left multiplying both sides of Eq. \ref{regression_task} by transpose of eigenvectors. Note that Eq. \ref{eq:ratnet2} is applicable when there is no other layers between the output $Y$ and the convolution operation. In contrast, Eq. \ref{eq:ratnet} can be used not only in regression task like Eq. \ref{regression_task}, but also in classification where there exist additional neural networks as described in Eq. \ref{classfication_task}. RationalNet has complexity $\mathcal{O}(| \mathcal{E}|)$ for Eq. \ref{eq:ratnet2} and $\mathcal{O}(| \mathcal{E}|^{3})$ for Eq. \ref{eq:ratnet}.

%
%

%

\subsection{Relaxed Remez Algorithm for initialization}
RationalNet is powerful at approximating a function. However, it is often stuck in a local optimum due to the neural network optimization,  which makes rational function not always better than the polynomial approximation. Remez exchange algorithm\cite{remez1934determination} is an iterative algorithm used to find simple approximations to functions. Nevertheless, the drawback of this minimax algorithm is that the order for optimum is unknown and the stopping condition is not often practical. To improve RationalNet's initialization, a relaxed Remez algorithm is proposed.  

As Theorem \textbf{24.1} (equioscillation characterization of best approximants, \cite{trefethen2013approximation}) states: Given the order of numerator(m) and denominator(n), and a real function $f$ that is continuous in [p, q], there exists a unique best rational approximation  $R^{*}$ defined as Eq. \ref{eq:rat_define}. This means that the $R^{*}$ optimizes the minimax error:  

\begin{equation}
\small
R^{*}= \arg\min_{R} \max_{x\in [p, q]} | f-R |.
\end{equation}

A rational function $R$ is equal to $R^{*}$ iff $f-R$ equioscillates between at least m+n+2 extreme points, or say the error function attains the absolute maximum value with alternating sign:

\begin{equation}
\small
f(x_{d})-R(x_{d})=(-1)^{d}E, d\in [0, m+n+1],
\label{eq:equioscillates}
\end{equation}where $d$ indicates the index of data point, and $E$ represents the max of residuals: $E=max_{x_{d}}|f(x_{d})-R(x_{d})|$. For rational function approximation, there is some nonlinearity because there will be a product of $E$ with $\phi_{j}$ in the equations. Hence, these equations need to be solved iteratively. The iteration formula can be defined by linearizing the equations:  

{\small
    \begin{equation}
    \sum ^{m}_{i=0}\psi_{i}x^{i}_{d}-\left[ f\left( x_{d}\right) -\left( -1\right) ^{d}E_{r}\right] \sum ^{k}_{j=1}\phi_{j}x^{j}_{d}=f\left( x_{d}\right) -\left( -1\right) ^{d}E_{r+1},
    \label{eq:remez_mat}
    \end{equation}
}where $r$ indicate the iteration index. Eq. \ref{eq:remez_mat} is obtained by neglecting nonlinear terms of the form $(E_{r}-E_{r+1})\phi_{j}x_{d}^{j}$ in Eq. \ref{eq:equioscillates}. This procedure can converge in a reasonable time (\cite{antiabook2004}). Expanding Eq. \ref{eq:remez_mat} for all sampled data points $x_{0}, x_{1}, ..., x_{d}$, it can be rewritten as:

{\tiny
    \begin{gather}
    \begin{bmatrix} 
    x_{0}^{0} & ... &x_{0}^{m} & (E_{r}-y_{0})x_{0}^{1}  & ... & (E_{r}-y_{0})x_{0}^{n} & (-1)^{0} \\
    x_{1}^{0} & ... &x_{1}^{m} & (E_{r}-y_{0})x_{1}^{1}  & ... & (E_{r}-y_{0})x_{1}^{n} & (-1)^{1} \\
    x_{2}^{0} & ... &x_{2}^{m} & (E_{r}-y_{0})x_{2}^{1}  & ... & (E_{r}-y_{0})x_{2}^{n} & (-1)^{2} \\
    ...\\
    x_{K}^{0} & ... &x_{K}^{m} & (E_{r}-y_{0})x_{K}^{1}  & ... & (E_{r}-y_{0})x_{K}^{n} & (-1)^{K} \\
    \end{bmatrix}
    \begin{bmatrix} 
    \psi_{0} \\ \psi_{1} \\ \psi_{2} \\ ..\\ \psi_{m} \\ \phi_{1} \\ \phi_{2} \\ \phi_{3} \\ .. \\ \phi_{n} \\ E_{r+1}
    \end{bmatrix}
    =
    \begin{bmatrix}
    y_{0}\\ y_{1}\\ ... \\ y_{K}
    \end{bmatrix},
    \label{eq:remez_mat_expand}
    \end{gather}
}where K=m+n+1. Starting from an assumed initial guess $E_{r=0}$， this set of linear equations can be solved for the unknown $\psi_{i}$, $\phi_{j}$ and $E_{r+1}$, when two successive values of $E_{r}$ are in satisfactory agreement such as $|E_{r+1}-E_{r}|$ is less than 1e-6. Constructing a rational function with new coefficients,  Remez computes the error residuals. If absolute value of the residuals $\delta$ are not great than $|E|$, the optimal coefficients are obtained. Otherwise, Remez calculates the roots of rational function and constructs a new set of control points by collecting the extremes in each interval of roots, and repeat the computation of Eq. \ref{eq:remez_mat} until residuals $\delta$ are not great than $|E|$. However, this stopping rule is not often satisfied, which makes the algorithm stuck in dead loop. Therefore, we add an iteration limit for avoid dead loop. The relaxed Remez algorithm could be summarized as follows:
\begin{myitemize}
    \item \textbf{Prepare training data}
    \begin{itemize}
        \item Specify the degree of interpolating rational function.
        \item Pick m + n + 2 points from the data points $X=\{x_{0}, x_{1}, ..., x_{m+n+1}\}$ with equal interval. Under this discrete setting, the distances between any neighbors are considered equal if the data distribution are dense
    \end{itemize}
    \item \textbf{Optimization by equioscillation constraint}
    \begin{itemize}
        \item Solve coefficients and $E$ by Eq. \ref{eq:remez_mat_expand}
        \item Form a new rational function $R$ with new coefficients
        \item Calculate residual errors
        \item Repeat until E converges or $|E_{r+1}-E_{r}|$ is less than 1e-6
    \end{itemize}
    \item \textbf{Check stopping rule}
    \begin{itemize}
        \item Calculate residual errors
        \item Stops if the absolute value of any residual is not numerically greater than $|E|$. 
        \item Otherwise, find the n+m+1 roots of the rational function, and find the points at which the error function attains its extreme value. Rerun the algorithm with this new set of training data from the second step.
    \end{itemize}
\end{myitemize}

We have considered an algorithm for obtaining minimax approximation when the function could be evaluated at any point inside the interval. In our case, the function is known only at a set of discrete points, since eigenvalues are not continuous. However, this problem is no essentially different form the continuous case if the set of points is rather dense in the target interval [p, q]. We simply assume that eigenvalue samples are dense enough, since we often normalized eigenvalues into the range [0,1], several hundreds of points are thereby sufficient. For example, our smallest size of the synthetic graph consists of 500 nodes, so there are 500 eigenvalues distributed in [0, 1], which should be enough for approximation.

If the degree of rational function is large, then the system of equations could be ill-conditioned. Sometime, the linear system of equations \ref{eq:remez_mat_expand} is singular, which make the solution vector($\psi_{i}$, $\phi_{j}$, $E_{r+1}$) under-determined. We traverse all possible m/n pairs given the maximum of m and n. The relaxed algorithm discards any m/n if singular matrix error occurs. 

We found that the residuals $\delta$ are not smaller than $|E|$ for some m/n pairs, and the algorithm continues to output the same values. In such case, the algorithm stops if the maximum and minimum residuals ($\delta_{min, max}$) converge or they satisfy $\delta_{0,1,...,m+n+1}<|E|$.

\section{Algorithm and Theoretical analysis}\label{sec:algorithm}
This section elaborates algorithm details and analyzes its convergence rate on jump discontinuity.

\subsection{Algorithm description}

The Algorithm 1 first calculate graph Laplacian(line 1) and spectral decomposition(line 3), and convert vertex signal into spectral domain by inverse Fourier transform(line 2). From given m,n, algorithm 1 traverse all possible m/n pairs (line 4). Picking up m+n+1 points with equal intervals, the optimal error is calculated (line 10). After convergence, optimal m/n and ψ i /ϕ j are determined (line 11). Then algorithm calculates the residuals(line 13). If the stopping rule is not satisﬁed, decrease the order of denominator or numerator in turns and repeat the same process, otherwise, output the parameters of rational function.

With optimal parameters, graph convolution operation is calculated by rational approximation (line 19). Then we conduct typical neural networks optimization.

\begin{algorithm}[!h]
    \caption{RationalNet}
    \label{algo:fgan}
    \SetAlgoLined
    \KwIn{a graph $\mathcal{G}=\{\mathcal{V}, \mathcal{E}\}$,\\
        rational function order: m, \\
        graph signal on nodes: Y(i), i $\in$ {$1, 2, ..., |\mathcal{V}|$} }
    \KwOut{a rational function with parameters: $\psi_{i}$ and $phi_{i}$ }
    compute graph Laplacian: $\Ll=A-D$ \label{algo:graph_lap}\\
    compute spectral signal by graph Fourier transform: $\hat{Y}=\UT Y$ \label{algo:inv_fourier} \\
    perform eigen decomposition: $\Ll = \Uu \Lambda \UT$ \label{algo:spectral_decomp}\\
    n $\leftarrow$ m    \label{algo:mn}\\
    // initialize parameters by a relexed Remez \\
    \Repeat{$\delta$ convergence or $\delta_{min, max}<|E|$}{
        Pick m + n + 1 points ${x_{0}, x_{1},...,x_{m+n+1}}$ from full data $X$ with equal interval \\
        r = 0, $E_{r}$ = 0 \\ 
        \Repeat{$E_{r+1} - E_{r}$ convergence}{
            solve $\psi_{0\sim m}, \phi_{1\sim n}, E_{r+1}$ \Comment{Eq. \ref{eq:remez_mat} or \ref{eq:remez_mat_expand} } \label{algo:remez_standard}\\
        }\label{algo:remez_standard_converge}
        form a Pad\'{e} rational function $R_{\psi,\phi}$ with $\psi_{0\sim m}, \phi_{1\sim n}$\\
        compute residues $\delta_{d} = |\hat Y(d)-R(x_{d})|    $\label{algo:compute_residual}\\
        m $\leftarrow$ m-1 or n $\leftarrow$ n-1 in turns.\\
    }
    // initialize a Pad\'{e} rational function with the above coefficients \\
    \Repeat{$\mathbf{MSE}$ converges}{
         form a Pad\'{e} rational function $R_{\psi,\phi}$ with $\psi_{0\sim m}, \phi_{1\sim n}$ obtained in the above repeat loop\\
        $R(\Ll)x=\Pp(\Ll)\Qq(\Ll)^{-1}x$ \Comment{Eq. \ref{eq:ratnet} or \ref{eq:ratnet2}} \label{algo:ratnet}\\
        $\theta=\{\psi_{i}, \phi_{j}\}$\\
        compute the mean error function  $\mathcal{L}=\mathbf{MSE}(R(x)-Y) $\\
        compute derivatives to update parameters: $ \theta \leftarrow \theta + \beta\nabla_{\theta}\mathcal{L}$, where $\beta$ is learning rate
    }
\end{algorithm}

\subsection{Theoretical analysis}

This section first represents jump discontinuity using a function(Eq. \ref{eq:jump_discontinuity}). Then convergence rate of rational function on jump discontinuity is analyzed(Theorem \ref{the:converge_rat}). With the help a Lemma \ref{lemma:converge_rat}, we prove Theorem \ref{the:converge_rat}. Similarly, convergence rate of polynomial function(Theorem \ref{the:poly_converge}) is also provided.

We found that $f_{\sigma=1}=a|x|+bx$ and $f_{\sigma=2}=a\frac{|x|}{x}+bx$ can characterize single jump discontinuity. For example, when a=b=1/2 and $\sigma$=0, $f_{1,2}$ is ReLU function. It is $sign(x)$ when a=1 and b=1, and $\sigma$=1. Thus, $f_{1,2}$ rotates or change the angle between two lines at jump discontinuity based on $|x|$ and $x$. These two functions can be rewritten in an uniform formula:
\begin{equation}
\small
\label{eq:jump_discontinuity}
f_{1,2}=a\frac{|x|}{x^{\sigma\in\{0,1\}}}+bx
\end{equation}where $a, b\in \mathbb{R}$. \begin{theorem}[convergence rate of rational approximation on jump discontinuity]
\label{the:converge_rat}
    Given n$\geq$5 and b$\geq$1, there exist a rational function $R_{n}(x)$ of degree n that satisfies\\
    \begin{equation*}
    \small
    \sup_{x\in [-c, c]}|f_{1, 2}-R_{n}(x)|\leq Ce^{-\sqrt{n}}.
    \end{equation*} 
    \label{the:rat_converge}
\end{theorem}In our proof of Theorem \ref{the:rat_converge}, for $n\in \mathbb{N}$, we follow \cite{newman1964rational} and define the Newman polynomial: $\nr(x):={\prod}_{i=1}^{n-1}(x+\alpha^{i}_{n})$, where $\alpha_{n}:= e^{-1/\sqrt{n}}$. To approximate jump discontinuity, define $A_{n}(x)$ as Newman approximation:$A_{n}(x):=x\frac{\nr(x)-\nr(-x)}{\nr(x)+\nr(-x)}$.

\begin{lemma}
\label{lemma:converge_rat}
    Given $n \in [5, \infty)\cap \mathbb{Z}$, $c\in [1,+\infty)$,  $\sigma\in\{0,1\}$
    \begin{equation}
    \small
    \sup_{x\in [-c, c]}\left|\frac{|x|}{x^{\sigma}}-\frac{cA_{n}(x/c)}{x^{\sigma}}\right|\leq 3ce^{-\sqrt{n}}.
    \label{rat-lemma}
    \end{equation}
\end{lemma}
\begin{proof}[proof for Lemma \ref{rat-lemma}]
    \textbf{If $\sigma$=0}, left of Eq. \ref{rat-lemma} is equivalent to 
    \begin{equation*}
    \small
    \begin{aligned}
    \left||x|-cA_{n}(\frac{x}{c})\right|=\left|c(|\frac{x}{c}|-A_{n}(\frac{x}{c}))\right|
    =c\left||\frac{x}{c}|-A_{n}(\frac{x}{c})\right|.
    \end{aligned}
    \end{equation*}
    
    Since $|\frac{x}{c}|$ and $A_{n}(\frac{x}{c})$ are both even, it suffices to consider the case when $0\leq x\leq c$. 
    
    \textbf{For $x \in [0, c\alpha_{n}^{n}=ce^{-\sqrt{n}}]$}, since $\nr(x)\geq\nr(-x)\geq0$ so that $\frac{a}{c}A_{n}(\frac{a}{c})\geq 0$:
    \begin{equation*}
    \small
    c\left||\frac{x}{c}|-A_{n}(\frac{a}{c})\right|\leq c\left||\frac{x}{c}|\right| = x \leq ce^{-\sqrt{n}} < 3ce^{-\sqrt{n}}.
    \end{equation*}
    
    \textbf{For $x \in (c\alpha_{n}^{n}=ce^{-\sqrt{n}}, c]$}:
    
    \begin{equation*}
    \small
    \begin{aligned}
    &c\left||\frac{x}{c}|-A_{n}(\frac{x}{c})\right|&&\\
    =&c\left|\frac{x}{c}-A_{n}(\frac{x}{c})\right|&&(\frac{x}{c}>0)\\
    =&c\left|\frac{x}{c}-\frac{x}{c}\frac{\nr(\frac{x}{c})-\nr(-\frac{x}{c})}{\nr(\frac{x}{c})+\nr(-\frac{x}{c})}\right|&&(\text{definition of }A_{n})\\
    =&2x\left|\frac{\nr(\frac{x}{c})}{\nr(-\frac{x}{c})}+1\right|^{-1}&&\\
    \leq&2c\frac{x}{c}\left[\left|\frac{\nr(\frac{x}{c})}{\nr(-\frac{x}{c})}\right|-|-1|\right]^{-1}&&(|a-b|\geq|a|-|b|)\\
    \leq&2c\left[\left|\frac{\nr(\frac{x}{c})}{\nr(-\frac{x}{c})}\right|-1\right]^{-1}&&(\frac{x}{c}\leq 1)\\
    \leq&\frac{2c}{e^{\sqrt{n}}-1}&& (\text{Lemma 3.2, Ch. 7, \cite{lorentz1996constructive}})\\
    \leq&\frac{2c}{e^{\sqrt{n}}-\frac{1}{3}e^{\sqrt{n}}}&&(\frac{e^{\sqrt{n}}}{3} \geq \frac{e^{\sqrt{5}}}{3} \approx \frac{3.19}{3}>1) \\
    =&3ce^{-\sqrt{n}}.&&
    \end{aligned}
    \end{equation*}
    
    \textbf{If $\sigma$=1}, 
    Following same procedure as $\sigma$=0, we have:
    \begin{equation*}
    \small
    \left|\frac{|x|}{x}-\frac{cA_{n}(x/c)}{x}\right|\leq 3ce^{-\sqrt{n}}.
    \end{equation*}
\end{proof}

\begin{proof}[proof for Theorem \ref{the:rat_converge}]
    Applying Lemma \ref{rat-lemma}:
    \begin{equation*}
    \small
    \begin{aligned}
    \left|f_{1}-R(x)\right|=&\left|(ax+b|x|)-(ax+bcA_{n}(\frac{x}{c}))\right|&&\\
    =&b\left||x|-cA_{n}(\frac{x}{c})\right|
    \leq 3bce^{-\sqrt{n}}.&&
    \end{aligned}
    \end{equation*}
    
    Similarly, we have:
    \begin{equation*}
    \small
    \begin{aligned}
    \left|f_{2}-R(x)\right|=&\left|(ax+b\frac{|x|}{x})-(ax+b\frac{cA_{n}(x/c)}{x}))\right|&&\\
    =&b\left|\frac{|x|}{x}-\frac{cA_{n}(x/c)}{x}\right|
    \leq 3bce^{-\sqrt{n}}.&&\\
    \end{aligned}
    \end{equation*}
    In sum, 
    \begin{equation*}
    \sup_{x\in [-c, c]}|f_{1, 2}-R_{n}(x)|\leq Ce^{-\sqrt{n}},
    \end{equation*}where C=3bc in Theorem \ref{the:rat_converge}
\end{proof}

By Bernstein's theorem\cite{achiezertheory}, polynomials can approximate a function with:
\begin{equation*}
\small
||x|-P_{n}(x)|\leq \frac{\beta}{n},
\end{equation*}where $P_{n}(x)$ is a polynomial function of degree of n, and $\beta\approx 2.801$\cite{Varga1985}.
Using the same settings for the rational function, we have a similar result for polynomials:

\begin{theorem}[convergence rate of polynomial approximation on jump discontinuity]
     Given $n \in [5, \infty)\cap \mathbb{Z}$, $c\in [1,+\infty)$,  $\sigma\in\{0,1\}$:
    \begin{equation*}
    \sup_{x\in [-c, c]}|f_{1, 2}-P_{n}(x)|\leq \frac{C\beta}{n},
    \end{equation*}where $P_{n}(x)$ is a polynomial function of degree n, and C=3bc.
    \label{the:poly_converge}
\end{theorem}In a nutshell, when the order is large or equal to 5, polynomial converges linearly regarding the order number, while rational function converges exponentially.

\section{Evaluation}\label{sec:evaluation}
This section elaborates evaluation with a detailed analysis of the behaviors of the proposed method on synthetic and real-world graphs.

\subsection{Training Setting and Baselines}

%
The input include a graph Laplacian $\Ll$, a graph signal residing on each vertex $Y$, and raw feature $x$ . In a nutshell, we aims at finding a function $f$ that satisfies $Y = f(\Ll, x)$:
\begin{equation}
\small
Y = \Uu \gt(\LBD) \UT x= \gt(\Ll) x,
\label{eq:regress}
\end{equation}where $\gt(\LBD)$ is set to be jump function such $|x|$ and $sign(x)$. In previous works, raw features and filtering signal are fed into the model to fit the graph signal. However, raw features have an impact on fitting graph signal, which distracts the analysis of filtering behaviors. As discussed in Section \ref{sec:method}, $x$ is set to be eigenvector $\Uu$ which is a uniform signal in spectral domain, so that we can focus on the behaviors of approximation methods.
We compare RationalNet(RNet) against several stat of art regression models:
\begin{itemize}
    \item Linear Regression(LR)
    \item Polynomial Regression(PR) \cite{stigler1974gergonne}
    \item Passive Aggressive Regression(PAR) \cite{crammer2006online}
    \item LASSO\cite{tibshirani1996regression}
    \item Epsilon-Support Vector Regression(SVR)\cite{smola2004tutorial}. Three kernels were applied: linear(L), polynomial(P) and RBF(R).
    \item Ridge Regression(RR)\cite{ng2004feature}
    \item Bayesian Ridge Regression(BR)\cite{mackay1992bayesian},  
    \item Automatic Relevance Determination(ARD)\cite{tipping2001sparse}
    \item Elastic Net(EN)\cite{zou2005regularization}
    \item Orthogonal Matching Pursuit(OMP)\cite{mallat1993matching}
    \item SGD Regression
    \item Huber Regression \cite{huber2011robust}
    \item ChebNet\cite{defferrard2016convolutional}. PolyNet is proposed by replacing Chebyshev polynomial with normal polynomial.
\end{itemize} 

\subsection{experiments on synthetic data}
To validate the effectiveness of RationalNet, we conduct a simulated test with synthetic data. The task is to recover signal on the vertexes, which is a regression problem. Specifically, we generated a graph comprised of several subgroups. The edge amount for each vertex in the same subgroup is randomly chosen between 0 and 8, while the links among different subgroups are sampled between 0 and 3. Experiments were conducted on a 500-node and a 1000-node graph. 
Two types of jump signals are fed into this network structure: $|x|$ and $sign(x)$. Since all eigenvalues are normalized into range [0, 1], jump points of  $|x|$ and $sign(x)$ are moved into the same range. Specifically, we used $|x-0.5|$ and $sign(x-0.5)$. Detailed results are shown in Table \ref{syn_500} and \ref{syn_1000}.

\begin{table}[hbt]
	\tiny
	\centering
	\begin{tabular}{l|cc|cc}
		\toprule \hline
		Method & S-ERR($|x|$) & V-ERR($|x|$) &  S-ERR($sign(x)$) & V-ERR($sign(x)$)  \\ \hline
		SVR-R & .0044$\pm$.0000 & .0043$\pm$.0000 &.3840$\pm$.0000&.2573$\pm$.0000 \\
		SVR-L & .0165$\pm$.0000 & .0111$\pm$.0000 &.3218$\pm$.0000&.2799$\pm$.0000 \\
		SVR-P & .0179$\pm$.0000 & .0131$\pm$.0000 &.3587$\pm$.0000&.2573$\pm$.0000 \\
		LR & .0161$\pm$.0000 & .0110$\pm$.0000 &.3211$\pm$.0000&.2788$\pm$.0000 \\
		RR & .0160$\pm$.0000 & .0110$\pm$.0000 &.3199$\pm$.0000&.2786$\pm$.0000 \\
		LASSO & .0157$\pm$.0000 & .0137$\pm$.0000 &.5581$\pm$.0000&.5087$\pm$.0000 \\
		EN & .0157$\pm$.0000 & .0137$\pm$.0000 &.5969$\pm$.0000&.5438$\pm$.0000 \\
		OMP & .0161$\pm$.0000 & .0110$\pm$.0000 &.3211$\pm$.0000&.2788$\pm$.0000 \\
		BR & .0161$\pm$.0000 & .0110$\pm$.0000 &.3210$\pm$.0000&.2788$\pm$.0000 \\
		ARD & .0161$\pm$.0000 & .0110$\pm$.0000 &.3210$\pm$.0000&.2788$\pm$.0000 \\
		SGD & .0152$\pm$.0000 & .0116$\pm$.0000 &.3191$\pm$.0001&.2795$\pm$.0003 \\
		PAR & .2871$\pm$.0997 & .2740$\pm$.1033 &1.0370$\pm$.8892&.9745$\pm$.8418 \\
		Huber & .0202$\pm$.0000 & .0123$\pm$.0000 &.3219$\pm$.0000&.2794$\pm$.0000 \\
		PolyFit & .0016$\pm$.0000 & .0010$\pm$.0000 &.2057$\pm$.0000&.1703$\pm$.0000 \\
		\hline
		ChebNet&.0021$\pm$.0000  & 1.1904$\pm$.0052  &.2058$\pm$ .0067 &.2084$\pm$ .0043\\
		PolyNet& .0016$\pm$.0000 &.0038$\pm$.0000  & .2011$\pm$.0095 & .2001$\pm$.0056 \\
		RNet & \textbf{5.2971e-6$\pm$1.2501e-8} &\textbf{.0001$\pm$.00000} &\textbf{.0103$\pm$.0001} &\textbf{.0153$\pm$.0006}\\ \hline
		\bottomrule 
	\end{tabular}
	\caption{1000-node graph test: s-err indicates error in spectral domain, while v-err represents error in vertex domain.}
	\label{syn_1000}
\end{table}

\begin{table}[hbt]
	\tiny
	\centering
	\begin{tabular}{l|cc|cc}
		\toprule \hline
		Method & S-ERR(|x|) & V-ERR(|x|) &  S-ERR($sign(x)$) & V-ERR($sign(x)$)  \\ \hline
		SVR-R & .0043$\pm$.0000 & .0044$\pm$.0000 &.2691$\pm$.0000&.2867$\pm$.0000 \\
		SVR-L & .0148$\pm$.0000 & .0131$\pm$.0000 &.2612$\pm$.0000&.2748$\pm$.0000 \\
		SVR-P & .0137$\pm$.0000 & .0138$\pm$.0000 &.2784$\pm$.0000&.2875$\pm$.0000 \\
		LR & .0140$\pm$.0000 & .0130$\pm$.0000 &.2582$\pm$.0000&.2734$\pm$.0000 \\
		RR & .0140$\pm$.0000 & .0130$\pm$.0000 &.2579$\pm$.0000&.2741$\pm$.0000 \\
		LASSO & .0135$\pm$.0000 & .0137$\pm$.0000 &.4723$\pm$.0000&.4865$\pm$.0000 \\
		EN & .0135$\pm$.0000 & .0137$\pm$.0000 &.5260$\pm$.0000&.5374$\pm$.0000 \\
		OMP & .0140$\pm$.0000 & .0130$\pm$.0000 &.2582$\pm$.0000&.2734$\pm$.0000 \\
		BR & .0140$\pm$.0000 & .0130$\pm$.0000 &.2581$\pm$.0000&.2734$\pm$.0000 \\
		ARD & .0140$\pm$.0000 & .0130$\pm$.0000 &.2581$\pm$.0000&.2734$\pm$.0000 \\
		SGD & .0135$\pm$.0000 & .0138$\pm$.0000 &.2597$\pm$.0007&.2764$\pm$.0008 \\
		PAR & .4026$\pm$.3980 & .3982$\pm$.3954 &.7412$\pm$.5682&.7456$\pm$.5029 \\
		Huber & .0158$\pm$.0000 & .0135$\pm$.0000 &.2581$\pm$.0000&.2734$\pm$.0000 \\
		PolyFit & .0010$\pm$.0000 & .0011$\pm$.0000 &.1488$\pm$.0000&.1699$\pm$.0000 \\
		\hline
		ChebNet&.0044$\pm$.0000&.0044$\pm$.0000&.2025$\pm$.0000    & .2115$\pm$.0004\\
		PolyNet&.0016$\pm$.0000&.0016$\pm$.0000& .2059$\pm$.0000   & .2083$\pm$.0004  \\
		RNet &\textbf{.0001$\pm$.0000} & \textbf{.0001$\pm$.0000} & \textbf{.0108$\pm$.0001} & \textbf{.1479$\pm$.0001} \\  \hline
		\bottomrule
	\end{tabular}
	\caption{500-node graph test: s-err indicates error in spectral domain, while v-err represents error in vertex domain.}
	\label{syn_500}
\end{table}

In 1000-node graph test on  $|x|$(first two columns in Table \ref{syn_1000}), PolyFit achieved the second lowest MSE(0.0016 for S-ERR). PolyNet's MSE(0.0016) is the same as that of PolyFit, which shows the power of polynomial regression. Chebyshev polynomial(ChebNet) dose not improve PolyNet, which implies that neural network might approximate the best coefficients of polynomials no matter what type of polynomial is used. LR, RR, LASSO, EN, OMP, BR, ARD, SGD SVR(L/P) performed at the same level(0.0015-0.0018). Our method(5e-6) significantly outperformed all the baselines by a large margin. Both the errors in spectral domain and vertex domain show the advantage of RationalNet. The 
Similarly, PolyFit and PolyNet and SVR(R) performed better than all the baselines except RationalNet. Our method still achieves the lowest MSE(0.004619 for S-ERR). The 500-node graph experiment(Table \ref{syn_1000})) also demonstrates the effectiveness of RationalNet.

\begin{figure}[!hbtp]
	\includegraphics[width=3.4in]{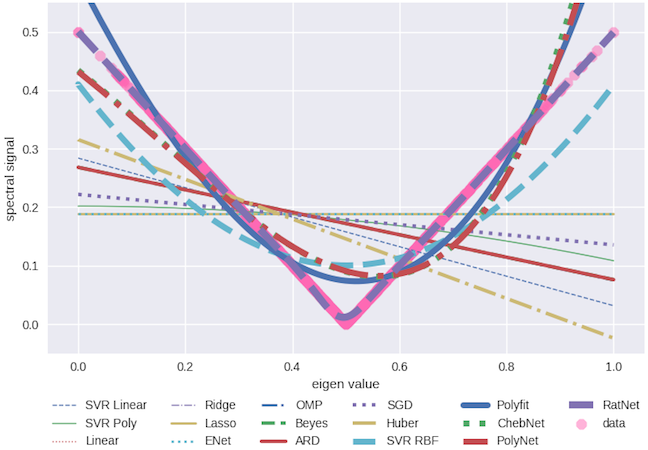}
	\caption{Regression performance comparison on $|x|$.}
	\label{abs_plot}
\end{figure}
Regression behaviors on synthetic data is shown in Fig. \ref{abs_plot} and \ref{sign_plot}. Methods (SVR(L), Ridge, OMP, LASSO, Linear regression, ENet, ARD, Huber, etc.) fitted the  $|x|$(Fig. \ref{abs_plot}) using a straight Line, while better baselines(SVR(R)), PolyFit, PolyNet, ChebNet) approximate the function with curves. RationalNet almost overlapped with the target function which makes its MSE very small(5e-6). Similarly, in Fig.  \ref{sign_plot}, the methods using straight lines(SVR(L), Ridge, OMP, SGD, LASSO, BR, Huber LR, ENet, ARD) performed relatively bad. Fitting with curves, PolyFit, PolyNet, ChebNet improved the performance by a large margin. Similarly, RationalNet overlapped the signal and achieved the lowest error score.
\begin{figure}[!hbtp]
	\includegraphics[width=3.4in]{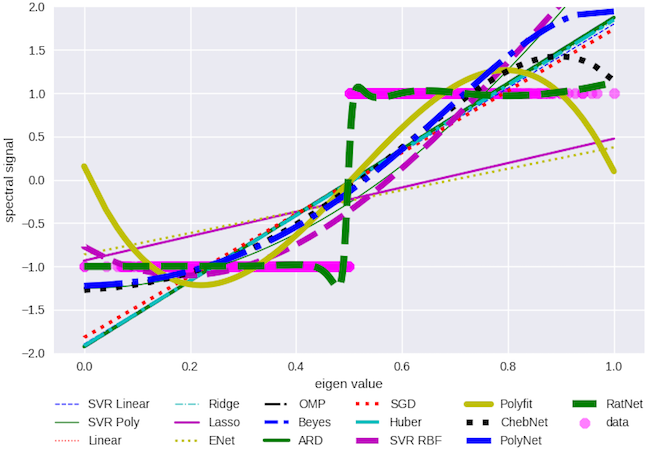}
	\caption{Regression performance comparison on $sign(x)$.}
	\label{sign_plot}
\end{figure}
\begin{figure*}[h]
	\centering
	\begin{subfigure}
		\centering
		\includegraphics[width=1.35in]{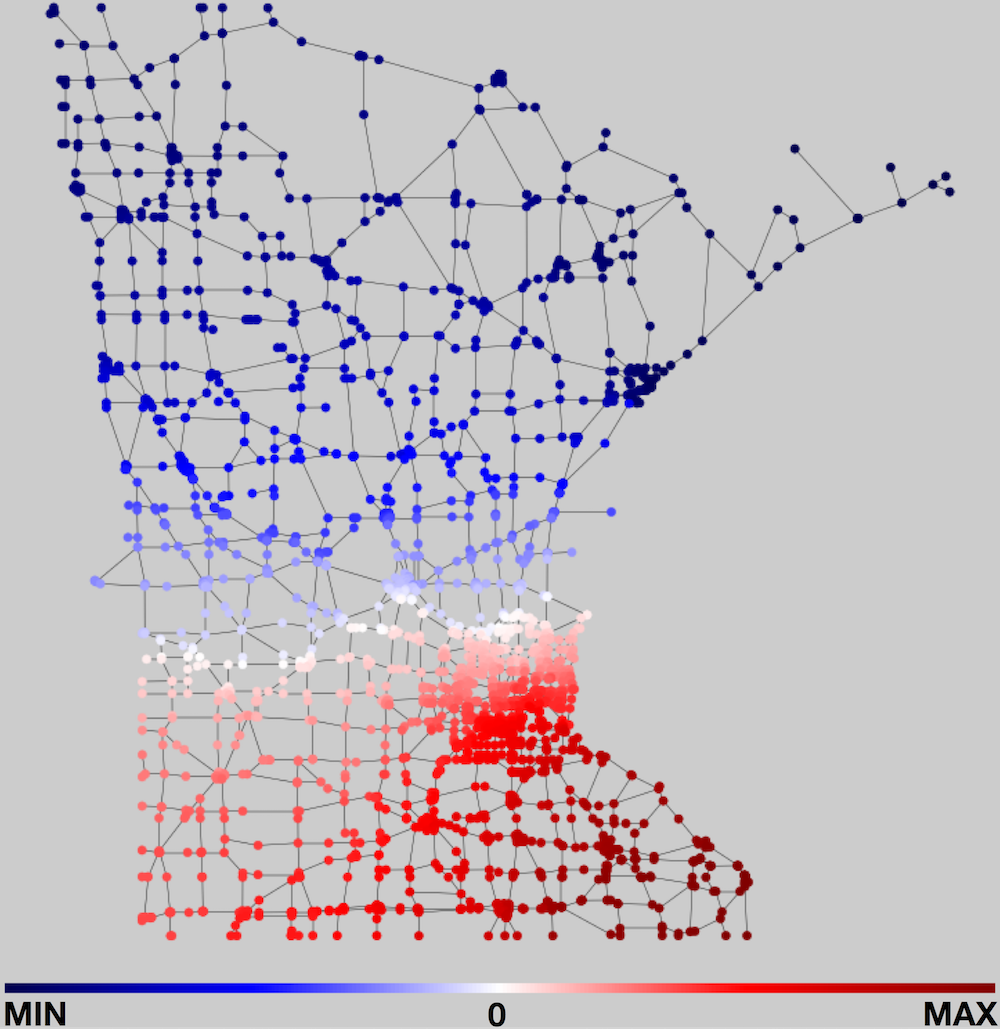}
	\end{subfigure}
	\begin{subfigure}
		\centering
		\includegraphics[width=1.35in]{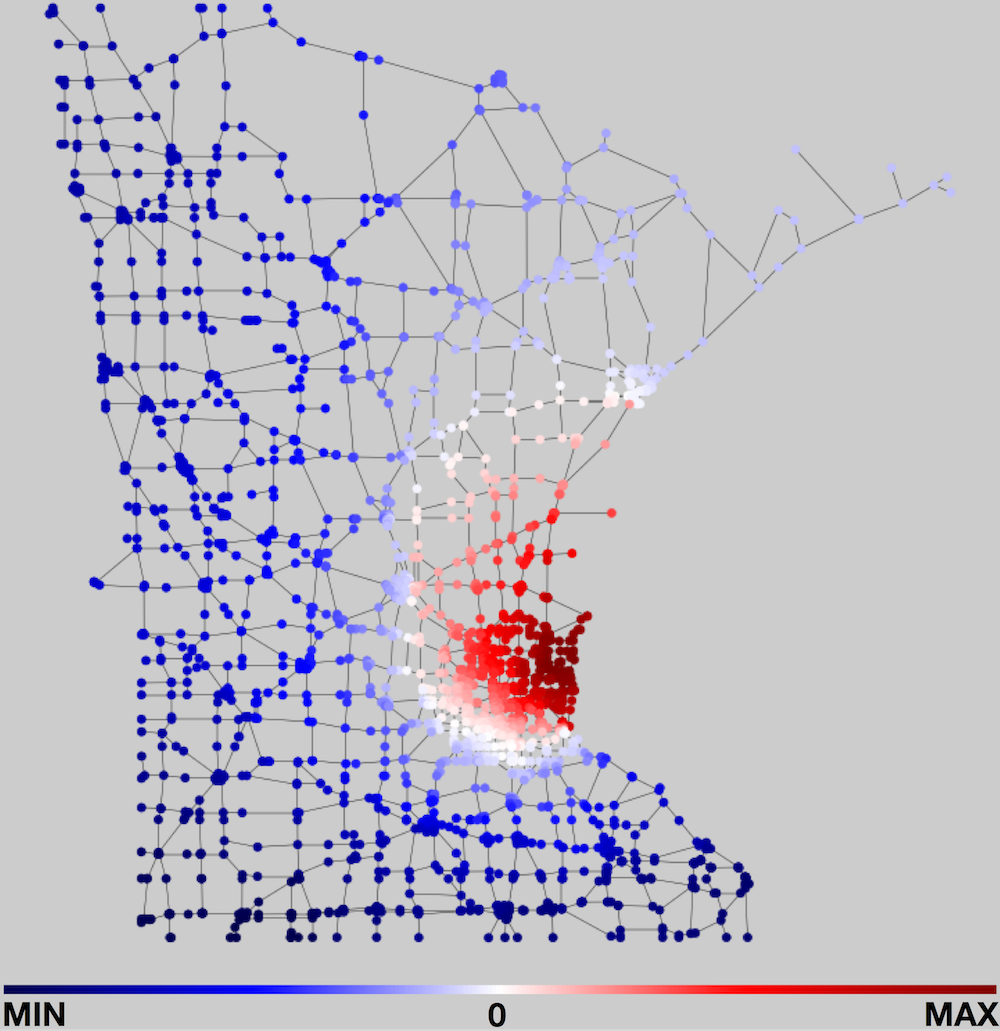}
	\end{subfigure}
	\begin{subfigure}
		\centering
		\includegraphics[width=1.35in]{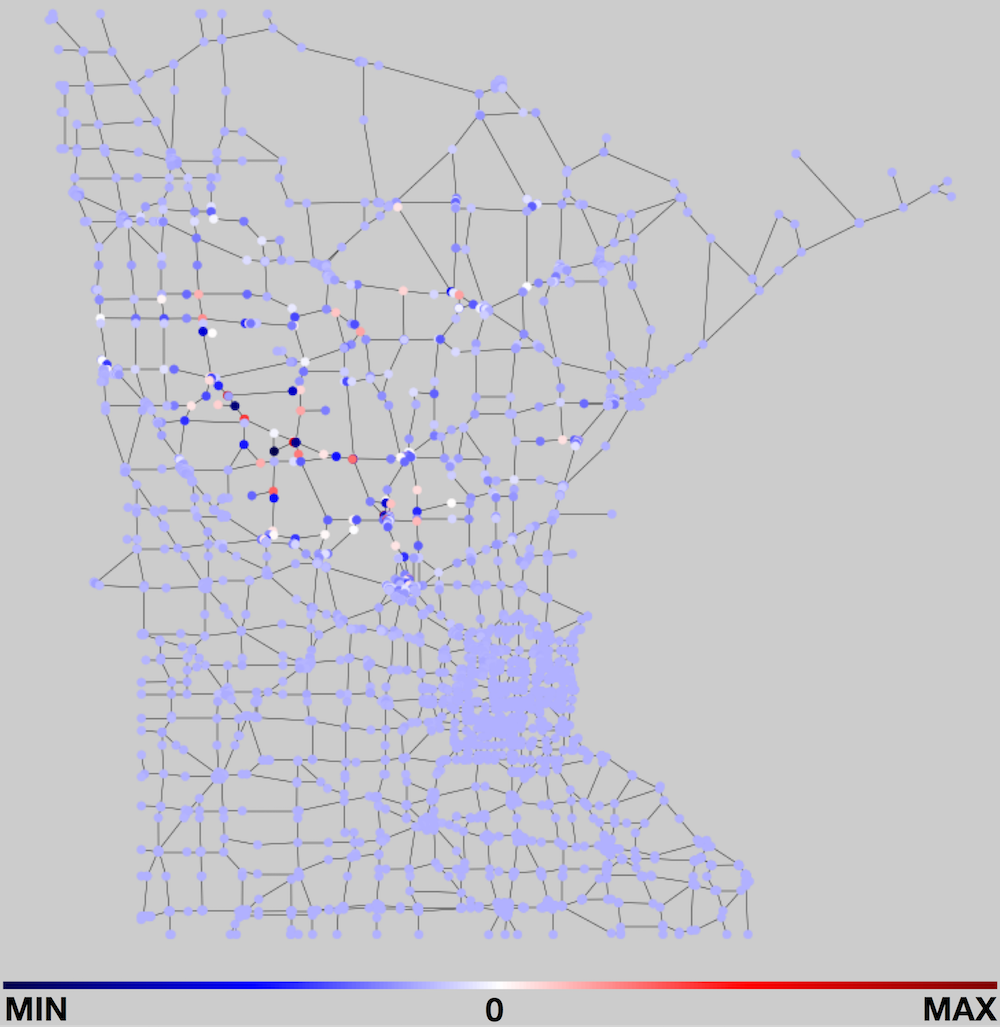}
	\end{subfigure}
	\begin{subfigure}
		\centering
		\includegraphics[width=1.35in]{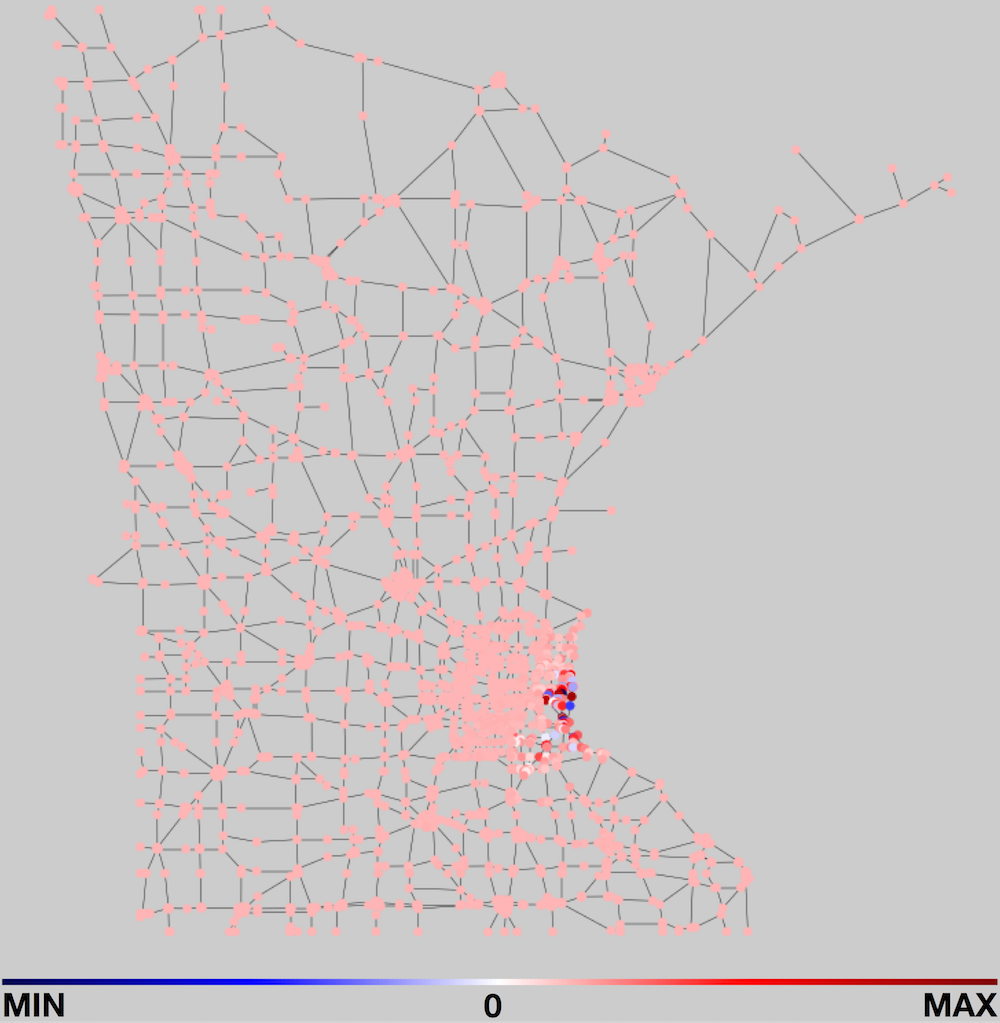}
	\end{subfigure}
	\begin{subfigure}
		\centering
		\includegraphics[width=1.35in]{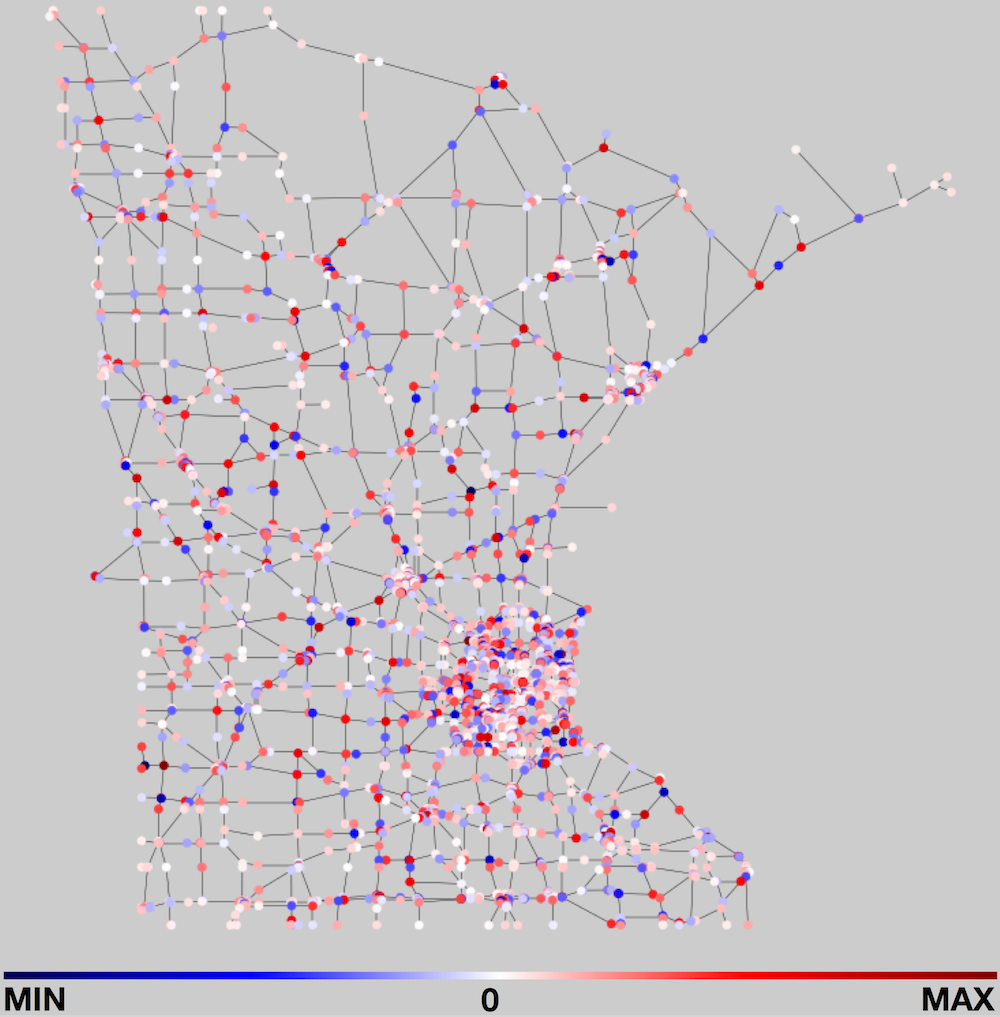}
	\end{subfigure}
	\begin{subfigure}
		\centering
		\includegraphics[width=1.355in]{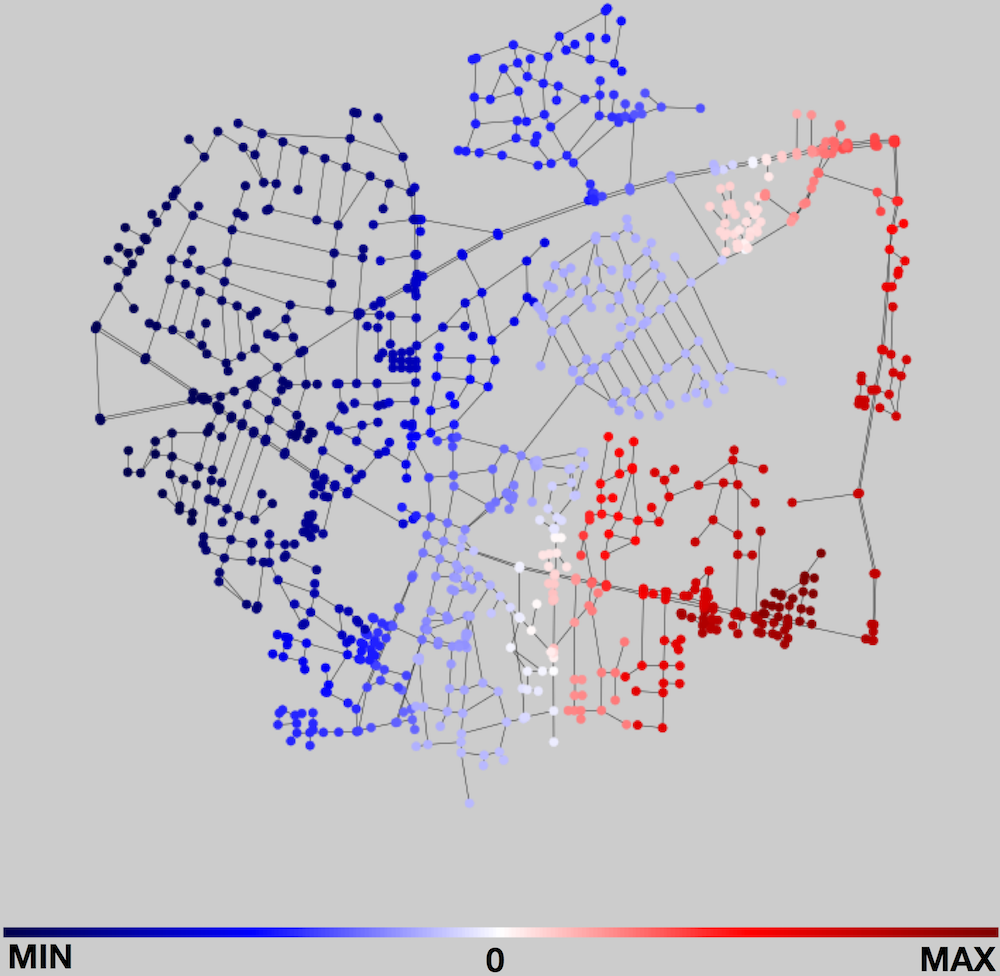}
	\end{subfigure}
	\begin{subfigure}
		\centering
		\includegraphics[width=1.355in]{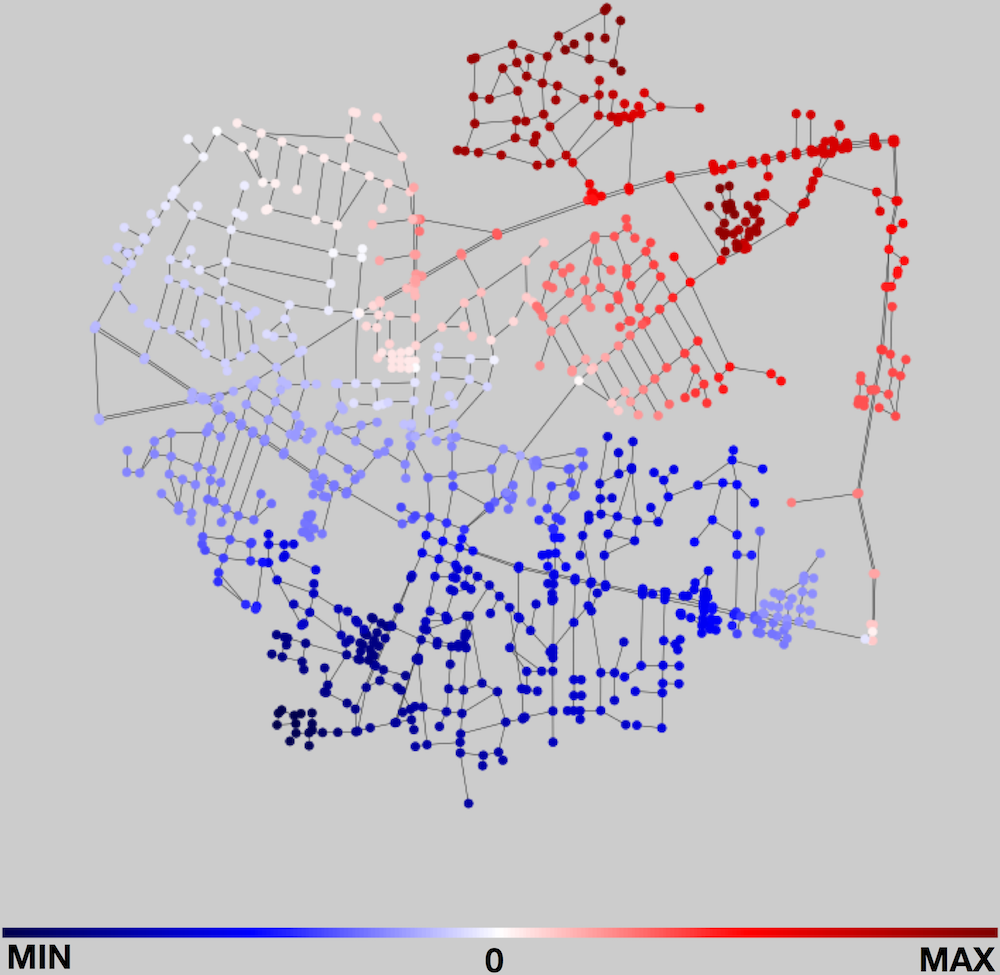}
	\end{subfigure}
	\begin{subfigure}
		\centering
		\includegraphics[width=1.355in]{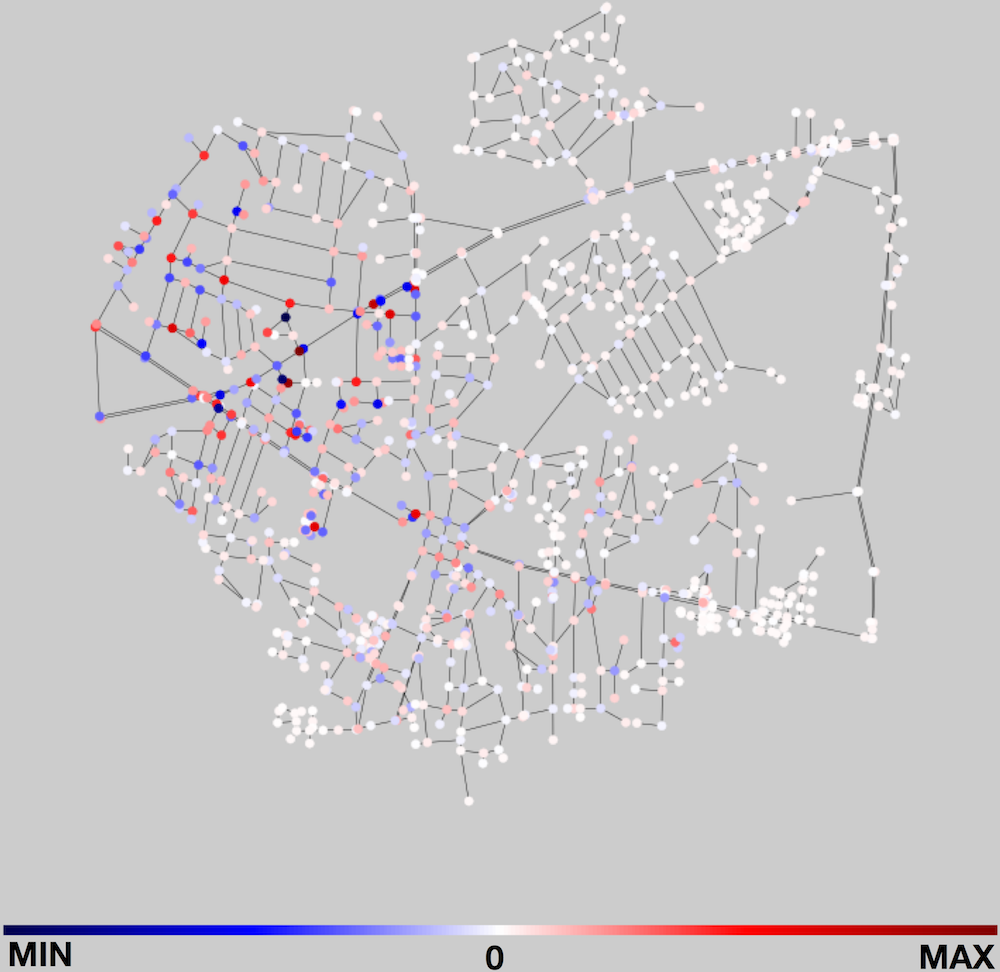}
	\end{subfigure}
	\begin{subfigure}
		\centering
		\includegraphics[width=1.355in]{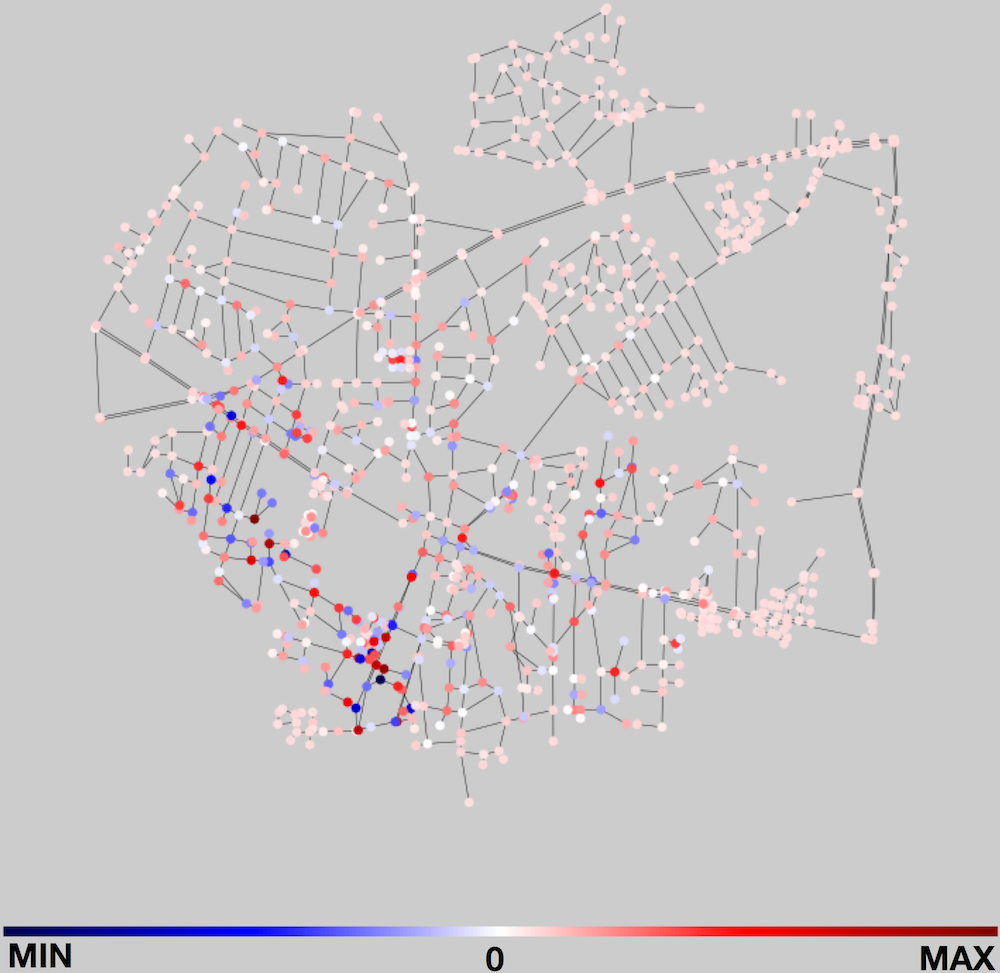}
	\end{subfigure}
	\begin{subfigure}
		\centering
		\includegraphics[width=1.355in]{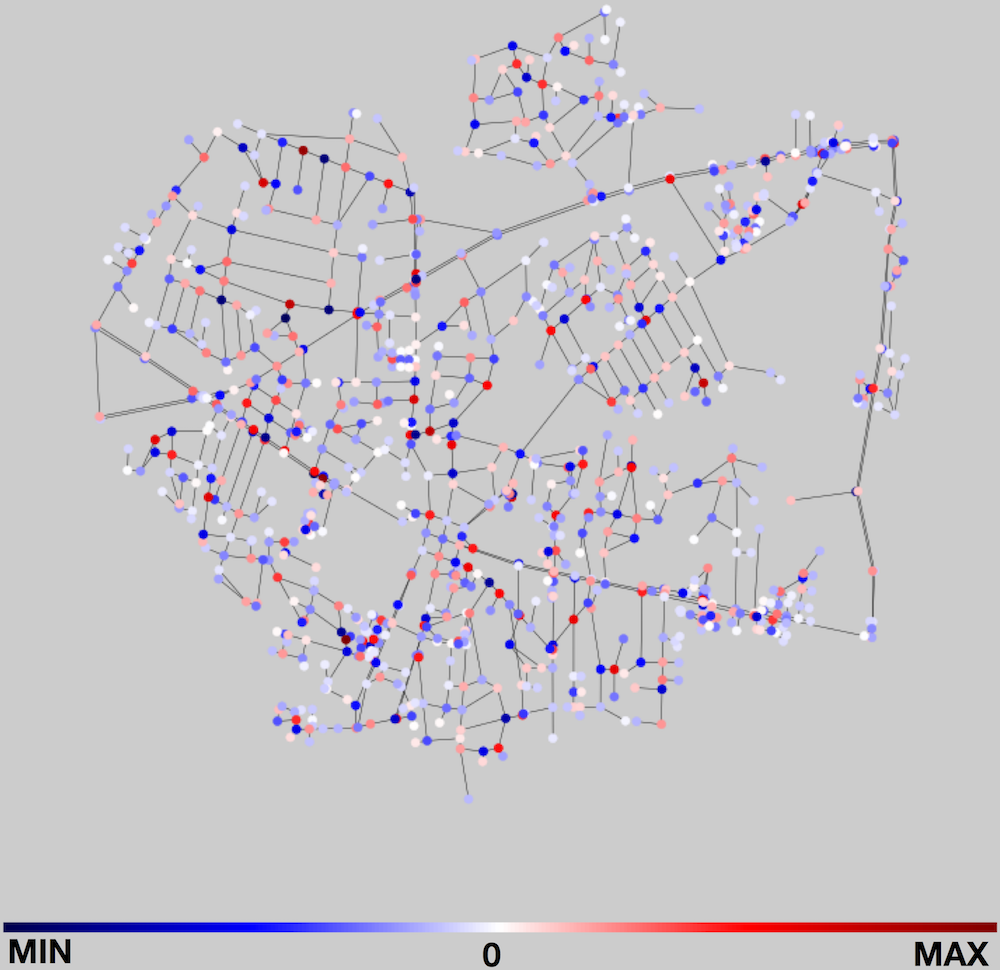}
	\end{subfigure}
	\caption{\textbf{Top line}: Minnesota road network. From left to right are $\varphi_{1}$,$\varphi_{2}$,$\varphi_{1001}$,$\varphi_{1002}$, and $\zeta$. $\mathcal{E}_{\varphi_{1}}$=0.00084， $\mathcal{E}_{\varphi_{2}}$=0.00207， $\mathcal{E}_{\varphi_{2001}}$=4.03932， $\mathcal{E}_{\varphi_{2002}}$=4.04661, $\mathcal{E}_{\zeta}$=15384.10112.
	\textbf{Bottom line}: Fairfax road network. From left to right are $\varphi_{1}$,$\varphi_{2}$,$\varphi_{701}$,$\varphi_{702}$, and $\zeta$. $\mathcal{E}_{\varphi_{1}}$=0.00250， $\mathcal{E}_{\varphi_{2}}$=0.00296， $\mathcal{E}_{\varphi_{701}}$=3.82741， $\mathcal{E}_{\varphi_{702}}$=3.81540, $\mathcal{E}_{\zeta}$=6376.54224
	}
	\label{minne_eig}
\end{figure*}

\begin{table}[hbt]
	\centering
	\begin{tabular}{l|c|c}
		\toprule \hline
		&  $|x|$ & $sign(x)$ \\ \hline
		Remez (spectral)  & 4.531041e-6 &  0.057233\\ 
		Remez (vertex)  &  0.000105 &  0.109641\\ \hline
		RNet w/o Remez (spectral) & 0.000145 &  1.364790\\ 
		RNet w/o Remez (vertex) & 0.000252 &  0.887602  \\ \hline
		RNet w/ Remez (spectral)  & 1.981569e-6 &  0.010645\\ 
		RNet w/ Remez (vertex)  & 9.569891e-5 &  0.093268\\ \hline
		Improved (spectral)  & 56.26\% &  81.39\%  \\ 
		Improved (vertex)  & 9.37\% &  14.93\% \\ \hline
		\bottomrule
	\end{tabular}
	\caption{Remez and RationalNet on 1000-node graph: MSE improvement in spectral and vertex domain}
	\label{remez_init}
\end{table}
Since RationalNet initializes parameters by a relaxed Remez algorithm, we analyze the performance of neural networks and Remez respectively. As shown in Table \ref{remez_init}, the first two rows show the MSE of Remez only. Compared with Remez algorithm, RationalNet without Remez initialization(3rd and 4th lines) performed badly. On the contrary, RationalNet with Remez initialization(5th and 6th lines) improved the Remez by 56.26\% and 81.39\%(7th line) for $|x|$ and $sign(x)$ separately in spectral domain, which also reduces their MSE in vertex domain by 9.37\% and 14.93\%(8th line). This result illustrates that Remez and RationalNet cannot find the optimum independently. Therefore, it is reasonable to integrate these two methods for optimizing the coefficients.

%
%
%
%
%

\subsection{Case study on real-world scenario}
In this section, we study a traffic congestion signal on Minnesota state-level road network \footnote{https://www.cise.ufl.edu/research/sparse/matrices/Gleich/minnesota.html} and Fairfax county-level road network VA\footnote{https://github.com/gboeing/osmnx}\cite{boeing2017osmnx}. Specifically, the signal is a high-pass filtering which can be written as $\zeta=\frac{sign(x-0.5)+1}{2}$ in Fourier domain. $\zeta$ is a threshold function sets the output to 0 when normalized eigenvalues $\in [0,0.5)$, and 1 for $\in (0.5, 1]$. Therefore, this function filter out signal of low frequency. 
The physical meaning of the convolutional operation is a weight function that chooses the eigenbasis($\varphi_{i}$) to fit the traffic signal $Y$. 
\begin{table}[]
	\scriptsize
	\centering
	\begin{tabular}{l|cc|cc}
		\toprule \hline
		Method & S-ERR(FF) & V-ERR(FF) &  S-ERR(MI) & V-ERR(MI)  \\ \hline
		SVR-R  &.0364$\pm$.0000&.0406$\pm$.0000  &.0393$\pm$.0000&.0358$\pm$.0000 \\
		SVR-L &.0652$\pm$.0000&.0599$\pm$.0000 &.0670$\pm$.0000&.0627$\pm$.0000 \\
		SVR-P &.1226$\pm$.0000&.1014$\pm$.0000 &.0518$\pm$.0000&.0499$\pm$.0000 \\
		LR &.0640$\pm$.0000&.0595$\pm$.0000  &.0662$\pm$.0000&.0621$\pm$.0000 \\
		RR &.0639$\pm$.0000&.0595$\pm$.0000 &.0662$\pm$.0000&.0621$\pm$.0000 \\
		LASSO &.2026$\pm$.0000&.2030$\pm$.0000 &.2141$\pm$.0000&.2138$\pm$.0000 \\
		EN &.1595$\pm$.0000&.1594$\pm$.0000 &.1609$\pm$.0000&.1592$\pm$.0000 \\
		OMP &.0640$\pm$.0000&.0595$\pm$.0000 &.0662$\pm$.0000&.0621$\pm$.0000 \\
		BR &.0640$\pm$.0000&.0595$\pm$.0000 &.0662$\pm$.0000&.0621$\pm$.0000 \\
		ARD  &.0640$\pm$.0000&.0595$\pm$.0000 &.0662$\pm$.0000&.0621$\pm$.0000 \\
		SGD &.0639$\pm$.0001&.0598$\pm$.0000  &.0664$\pm$.0000&.0622$\pm$.0000 \\
		PAR &.4960$\pm$.3273&.4948$\pm$.3200 &.4255$\pm$.4575&.4222$\pm$.4588 \\
		Huber &.0646$\pm$.0000&.0597$\pm$.0000 &.0666$\pm$.0000&.0624$\pm$.0000 \\
		PolyFit &.0346$\pm$.0000&.0382$\pm$.0000 &.0384$\pm$.0000&.0346$\pm$.0000 \\
		\hline
		ChebNet&.0468$\pm$.0006& .0468$\pm$.0006 & .2336$\pm$.0094&.2336$\pm$.0094\\
		PolyNet& .0468$\pm$.0006& .0468$\pm$.0006 & .0490$\pm$.0049&.0490$\pm$.0009 \\
		RNet & \textbf{.0064$\pm$.0007} &\textbf{.0064$\pm$.0007} &\textbf{.0046$\pm$.0012}&\textbf{.0046$\pm$.0006}\\ \hline
		\bottomrule
	\end{tabular}
	\caption{Regression comparison on Fairfax(FF) and Minnesota(MI) road networks. s-err indicates error in spectral domain, while v-err represents error in vertex domain.}
	\label{minne_mse}
\end{table}
The top line of Fig. \ref{minne_eig} shows several examples in eigen space of Minnesota road networks. First two sub figures are the 2nd and 3rd eigenvector $\varphi_{1}, \varphi_{2}$ on vertex domain: $\varphi_{1}$ emphasizes the south of Minnesota(red area), while $\varphi_{2}$ highlights the capital St. Paul and its biggest city Minneapolis. Note that the 1st eigenvector $\varphi_{0}$ is a constant vector for any connected graph. $\varphi_{1}, \varphi_{2}$ correspond to $\lambda_{1}, \lambda_{1}$, which represent the first two lowest frequencies. As these figures show, low frequencies represent smooth signals, which means that the neighbors of each node are likely to have similar signal value. By contrast, high-frequency basis captures non-smooth component as the 3rd and 4th sub figures show: signal values vary frequently in some areas. Combining top 50\% high-frequency eigenbasis, the last sub figure shows the $\zeta$ signal on the graph. In addition, the degree of non-smoothness of signal regard graph structure can be evaluated quantitatively by Dirichlet energy(\cite{shuman2013emerging}): $\mathcal{E}_{\varphi_{i}}=\varphi_{i}^{\intercal}\Ll\varphi_{i}$. Dirichlet energy of examples is shown in the caption of Fig. \ref{minne_eig}. Eigenvectors of low frequency($\varphi_{1,2}$) are smooth, so their Dirichlet energies are low. While high-frequency eigenvectors are less smooth since their Dirichlet energy is higher(around 4.04). Summing up the top 50\% high frequencies, the Dirichlet energy of $\zeta$ in the last sub figure is very large(15384.10). The bottom line of Fig. \ref{minne_eig} shows similar examples from Fairfax road networks. $\varphi_{1}$ highlights Fair City Mall(red area) and the road to this mall, while $\varphi_{2}$ underlines Fairfax Circle Shopping Center and a residential neighborhoods nearby. Similarly, $\varphi_{701}$ and $\varphi_{702}$ show two non-smooth graph signals. Summing up top 50\% high frequencies, the 5th sub figure exhibits an extremely non-smooth signal. Characterizing non-smooth graph signal or high frequencies is not a trivial task. Therefore, approximating this high pass filtering is significantly challenging.
\begin{figure}[!hbtp]
	\includegraphics[width=3.5in]{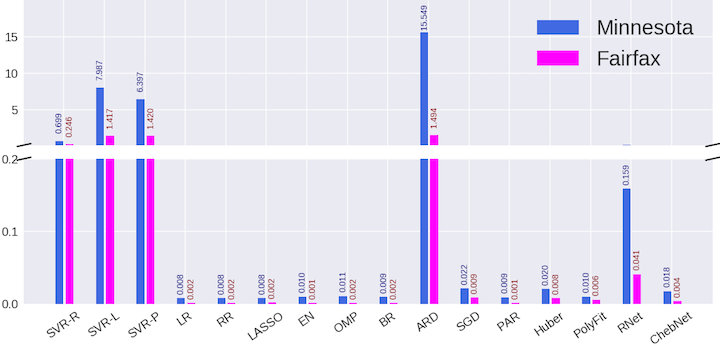}
	\caption{Comparison of average running time in seconds.}
	\label{time_cost}
\end{figure}
Table \ref{minne_mse} shows similar results as in synthetic experiments: The proposed method still performed much better(3e-5) than the baselines. PolyFit achieved the second best level(0.0008), ChebNet, PolyNet and SVR(RBF) are generally good(0.0039,0.0039 and 0.0055), this is probably because they fitted the target with curves. The methods using straight lines have highest level of MSE(around 0.01). The results on another dataset, Fairfax road network, also show that RationalNet has huge advantage beyond the baselines.

Fig. \ref{time_cost} shows the comparison of running time on two real-world networks. Minnesota dataset contains 2642 vertexes, while Fairfax network consists of 993 nodes. Most baseline methods are efficient such as LR, RR, LASSO, EN, OMP, BR, PAR. They finish computing within around 0.01 second on Minnesota graph and 0.002 second on Fairfax graph. SGD, PolyFit, and Huber only require 0.02 and 0.01 for Minnesota and Fairfax network respectively. SVR group performed slower, but they complete the calculation within 10 seconds for Minnesota graph and 2 seconds for Fairfax. ARD needs around 15 seconds and 1.4 seconds separately, which is the slowest baseline. Note that the number for RationalNet and ChebNet in Fig. \ref{time_cost} is the time for each iteration. RationalNet took 0.159 seconds for one update on Minnesota network, and 0.041 seconds on Fairfax network. In practice, RationalNet often converges within 300 iterations, which takes less than one minute for both datasets. Due to the complexity of computation, it is natural that RationalNet is slower than its counterpart ChebNet and several baselines. However, it shows that our algorithm can run reasonably fast in real-world datasets. 
Our case study on real-world graph justifies that RationalNet can accurately estimate the high pass filter within a reasonable time.

\section{Conclusion}\label{sec:conclusion}
In this paper, we have introduced a neural network model for graph signal recovering. To estimate jump discontinuity, a rational function is employed due to its powerful ability of approximation. The proposed method can avoid multiplication with the eigenvector matrix. With the help of a relaxed Remez algorithm, RationalNet can identify the optimal configuration. In theory, RationalNet obtains exponential convergence rate on jump signal, significantly fast than the polynomial-based approximation. Experiments on synthetic datasets suggest that the proposed RationalNet model is capable of model typical jump function accurately.

\section{Acknowledge}
Rongjie Lai’s research is supported in part by NSF DMS-1522645 and an NSF Career award DMS-1752934. Feng's research is supported in part by NSF IIS-1441479, NSF IIS-1815696, and an NSF CAREER award IIS-1750911. We gratefully acknowledge the support of NVIDIA Corporation with the donation of the Nvidia Titan V GPU used for this research.

\label{Bibliography}
\bibliographystyle{IEEEtran}
\bibliography{IEEEtran} 

\end{document}